\documentclass[letterpaper]{article} 
\usepackage{aaai2026}  
\usepackage{times}  
\usepackage{helvet}  
\usepackage{courier}  
\usepackage[hyphens]{url}  
\usepackage{graphicx} 
\usepackage{natbib}  
\usepackage{caption} 
\usepackage{subcaption} 
\usepackage{mathrsfs}
\usepackage{caption}
\usepackage{subfig}
\usepackage{bm}
\usepackage{algorithm}
\usepackage{algorithmic}

\usepackage{pifont}
\usepackage{multirow}
\usepackage{array}
\usepackage{amsthm}
\newtheorem{theorem}{Theorem}

\newtheorem{lemma}{Lemma}

\usepackage{amsmath}
\usepackage{amssymb}

\usepackage[utf8]{inputenc} 
\usepackage{url}            
\usepackage{booktabs}       
\usepackage{amsfonts}       
\usepackage{nicefrac}       
\usepackage{microtype}      
\usepackage{xcolor}         

\usepackage{booktabs}    
\usepackage{multirow}    
\usepackage{siunitx}     
\usepackage{adjustbox}   

\usepackage{algorithm}
\usepackage{algorithmic}

\usepackage{newfloat}
\usepackage{listings}
\DeclareCaptionStyle{ruled}{labelfont=normalfont,labelsep=colon,strut=off} 
\floatstyle{ruled}
\newfloat{listing}{tb}{lst}{}
\floatname{listing}{Listing}
\title{Enhancing Kernel Power $K$-means: Scalable and Robust Clustering with Random Fourier Features and Possibilistic Method}
\author{
    Yixi Chen\textsuperscript{\rm 1},
    Weixuan Liang\textsuperscript{\rm 1}\thanks{Corresponding author},
    Tianrui Liu\textsuperscript{\rm 1},
    Jun-Jie Huang\textsuperscript{\rm 1},
    Ao Li\textsuperscript{\rm 2},
    Xueling Zhu\textsuperscript{\rm 3},
    Xinwang Liu\textsuperscript{\rm 1}\thanks{Corresponding author}\\
}
\affiliations{
    \textsuperscript{\rm 1}{College of Computer Science and Technology, National University of Defense Technology, Changsha, China, 410073}\\
    \textsuperscript{\rm 2}{School of Computer Science and Technology, Harbin University of Science and Technology, Harbin, China, 150080}\\
    \textsuperscript{\rm 3}{Department of Radiology, Xiangya Hospital, Central South University, Changsha, China, 410008}\\

    chenyixi20@nudt.edu.cn,  weixuanliang@nudt.edu.cn, xinwangliu@nudt.edu.cn
}

\usepackage{bibentry}

\begin{document}
\maketitle

\begin{abstract}
Kernel power $k$-means (KPKM) leverages a family of means to mitigate local minima issues in kernel $k$-means. However, KPKM faces two key limitations: (1) the computational burden of the full kernel matrix restricts its use on extensive data, and (2) the lack of authentic centroid-sample assignment learning reduces its noise robustness. To overcome these challenges, we propose RFF-KPKM, introducing the first approximation theory for applying random Fourier features (RFF) to KPKM. RFF-KPKM employs RFF to generate efficient, low-dimensional feature maps, bypassing the need for the whole kernel matrix.
Crucially, we are the first to establish strong theoretical guarantees for this combination: (1) an excess risk bound of $\mathcal{O}(\sqrt{k^3/n})$, (2) strong consistency with membership values, and (3) a $(1+\varepsilon)$ relative error bound achievable using the RFF of dimension $\mathrm{poly}(\varepsilon^{-1}\log k)$. Furthermore, to improve robustness and the ability to learn multiple kernels, we propose IP-RFF-MKPKM, an improved possibilistic RFF-based multiple kernel power $k$-means. IP-RFF-MKPKM ensures the scalability of MKPKM via RFF and refines cluster assignments by combining the merits of the possibilistic membership and fuzzy membership. Experiments on large-scale datasets demonstrate the superior efficiency and clustering accuracy of the proposed methods compared to the state-of-the-art alternatives.
\end{abstract}


\section{Introduction}

Kernel $k$-means (KKM) stands as a cornerstone methodology in unsupervised learning paradigms, demonstrating broad applicability across various practical domains through its effective pattern recognition capabilities \cite{filippone2008survey,fengincremental,feng2025incremental}. 
However, KKM remains prone to sub-optimal local minima due to the non-convexity of the underlying objective function \cite{jain2010data,tang2012integrating,sun2021convex}.   
To address the problem, a series of works have focused on improving the non-convexity of the objective function of $k$-means. Instead of minimizing the sum of squared distances to the nearest centroid, convex clustering \cite{pelckmans2005convex} reduces the distance between data points and their convex combinations, resulting in a convex optimization problem with a unique optimal solution. In contrast, $k$-harmonic means (KHM) \cite{zhang1999k} minimizes the harmonic mean of distances to all centroids. 
Recently, Power $k$-means (PKM) \cite{xu2019power} generalizes KHM by replacing harmonic averaging with power means and incorporating an annealing strategy to decrease the power exponent during iterations gradually.
Building on this, Kernel Power $k$-means (KPKM) \cite{paul2022implicit} extends the framework to handle non-linearly separable data through kernelization. 

Although KPKM mitigates the tendency of KKM to fall into local optima, it 
cannot be directly applied to large-scale scenarios due to its quadratic time complexity in the number of instances $n$. Current clustering application areas increasingly involve large-scale datasets containing thousands or even millions of instances, such as social network analysis \cite{malliaros2013clustering}, bioinformatics \cite{zou2020sequence}, and financial market analysis \cite{kou2014evaluation,li2021integrated}. However, KPKM requires computing an $n\times n$ kernel matrix, leading to $\mathcal{O}(n^2k)$ complexity per iteration ($k$ denotes the number of clusters). Such computational demands become prohibitive when handling large-scale datasets in the scenarios above. 

To address the time complexity problem, we propose RFF-based kernel power $k$-means (RFF-KPKM). Specifically, we employ a novel technique known as random Fourier features (RFF) \cite{chitta2012efficient} to avoid computing the whole kernel matrix. RFF's advantage lies in its ability to directly construct an explicit mapping that approximates the kernel similarity via the inner product of the mapped data. 
RFF-KPKM significantly reduces the complexity of KPKM from $\mathcal{O}(n^2k)$ to $\mathcal{O}(n(k+d)D)$, where $d$ represents the data dimensionality and $D$ denotes the RFF dimensionality ($D \ll n$ typically). Beyond computational efficiency, our method is rigorously grounded in theoretical guarantees. We pioneer the first theoretical framework for approximating KPKM via RFF, establishing rigorous guarantees in three key dimensions. First, we establish an excess risk bound of $\mathcal{O}(\sqrt{k^{3}/n})$ for the approximation method (Theorem \ref{theorem: theorem1}), matching the optimal rate of the original method. 
This preservation of statistical guarantees demonstrates that our complexity reduction does not compromise theoretical reliability. 
Second, we establish a strong consistency guarantee for the established method (Theorem \ref{theorem:consistency}), demonstrating that as $n\to \infty$, RFF dimensionality $D\to \infty$ and power parameter $s\to -\infty$, the membership value converges to the globally optimal solution of the expected version of KKM with any constant probability. Lastly, we discuss how many RFF are needed to preserve accuracy. We obtain that when $n\to \infty$, $s\to -\infty$ and $D = \mathrm{poly}(\varepsilon^{-1}\log k)$, the membership value obtained by RFF-KPKM will produce a $(1+\varepsilon)$ relative error on the objective of expected KKM (Theorem \ref{theorem:approx consistency}).

Besides time complexity, another issue KPKM faces is noise sensitivity. Similar to fuzzy $c$-means (FCM) \cite{bezdek1984fcm}, the cluster centroids of KPKM are affected by the fuzzy membership value of the data. However, the fuzzy membership value $w_{ij}$ of data point $\boldsymbol{x}_{i}$ in cluster $\mathcal{C}_j$ is determined by the distance between $\boldsymbol{x}_{i}$ to all $k$ cluster centroids, while \citet{krishnapuram1993possibilistic} propose that such memberships may cause a noise problem since they doot reflect the absolute degree of typicality (or ``belonging'') of a point in a cluster. Possibilistic $c$-means (PCM) \cite{krishnapuram1993possibilistic} was proposed to learn the possibilistic membership of data, which is determined solely by the distance of a point from a cluster, thereby reflecting the absolute degree of typicality and resolving the noise problem. 


To address the noise problem in KPKM, we integrate our RFF-KPKM framework with the robust clustering paradigm of PCM. 
Specifically, we replace the distance function in RFF-KPKM with the objective function of PCM to learn the possibilistic membership of data. Also, by calculating the partial derivative of the objective function, we can obtain the fuzzy membership of the data. The product of possibilistic membership and fuzzy membership forms the final membership value. By integrating fuzzy and possibility methods, our method enhances KPKM's robustness and prevents the identical cluster issues commonly encountered in PCM \cite{zhang2004improved}. Furthermore, to extend our method to multi-view scenarios, we combine it with the kernel combination method in multiple kernel learning \cite{gonen2011multiple}, constructing reweighted RFF to learn the importance of mapping features in different views.
Our method is termed improved possibilistic RFF-based multiple kernel power $k$-means (IP-RFF-MKPK). 
To obtain the optimal solution, we conduct a novel majority minimization (MM) algorithm \cite{mairal2013stochastic}.

This work makes the following primary contributions:

 1) We propose RFF-KPKM to accelerate KPKM. Theoretically, we obtain that RFF with $\Omega(nd/k)$ dimensions achieves the excess risk bound $\mathcal{O} (\sqrt{k^{3}/n})$. To the best of our knowledge, this is the first bound on excess risk for the approximate kernel clustering based on RFF. 
 
 2) We further establish the strong consistency of the membership matrix produced by our algorithm when $n\to \infty$, RFF dimensionality $D\to \infty$, and power parameter $s\to -\infty$. When $D$ is finite, we show that $\mathrm{poly}(\varepsilon^{-1}\log k)$ RFF suffices to achieve the $(1+\varepsilon)$ relative error bound.
 
 3) We propose IP-RFF-MKPKM to solve the noise sensitivity problem in KPKM and extend RFF-KPKM to multi-view settings. The proposed method is more robust than MKPKM and is scalable to large-scale schemes. Experiments on large-scale datasets validate the computational efficiency and clustering accuracy of our methods.



\section{Notations}
To prevent ambiguity, boldface uppercase and lowercase characters are used to represent matrices and vectors. Specifically, $\mathbf{A}$ represents a matrix and  $\boldsymbol{a}$ represents a vector. The component of them is denoted by $A_{ij}$ or $a_{i}$. We denote with $\mathcal{X}\subset \mathbb{R}^{d}$ the sample space and with $\rho$ the corresponding data distribution. The training set $S_n=\{\boldsymbol{x}_{i}\}_{i=1}^n\subset\mathcal{X}$ is drawn i.i.d from $\rho$ with compact support $C\subset \mathbb{R}^{d}$. The empirical distribution $\rho_{n}$ is defined as $\rho_n(\boldsymbol{x}) = \frac{1}{n}$ if $\boldsymbol{x}\in S_n$, otherwise 0. In this paper, we adopt the Gaussian kernel as the default choice throughout our analysis. We denote the Gaussian kernel function as $k(\boldsymbol{x}_{i},\boldsymbol{x}_{j})=\exp(-\|\boldsymbol{x}_i-\boldsymbol{x}_j\|^2/2\sigma^2)$ where $\sigma$ is bandwidth parameter. As proposed in \cite{aronszajn1950theory}, there exist a feature mapping $\phi:\mathcal{X}\to\mathcal{H}$ such that $\forall \boldsymbol{x}_i, \boldsymbol{x}_j \in \mathcal{X}$, $k(\boldsymbol{x}_i,\boldsymbol{x}_j)=\langle\phi(\boldsymbol{x}_i),\phi(\boldsymbol{x}_j)\rangle $, where $\mathcal{H}$ is a Hilbert space.
We denote with $\mathrm{conv}(A)$ the closed convex hull of a set $A$. Finally, we assume $S_n$ is bounded.
\section{Related Work} 
\subsection{Kernel Power $K$-means}
KPKM \cite{paul2022implicit} embeds the kernel $k$-means problem within a continuum of well-posed surrogate problems. These intermediate formulations employ progressively smoothed objectives, which tend to guide clustering toward the global minimum of the original kernel $k$-means optimization landscape. The formula of this problem is
\begin{equation}
f_s(\boldsymbol{\Theta})=\sum_{i=1}^n M_s(\|\phi(\boldsymbol{x}_i)-\boldsymbol{\theta}_1\|^2,\ldots,\|\phi(\boldsymbol{x}_i)-\boldsymbol{\theta}_k\|^2),
\end{equation}
where 
$M_s(\boldsymbol{y}) = ( \frac{1}{k} \sum_{i=1}^{k} y_i^s )^{1/s}$ for a vector $\boldsymbol{y}$ and is called power means. 
KPKM seeks to minimize $f_s$ iteratively while sending $s\to -\infty$. As stated in \cite{paul2022implicit}, when s is small, the objective of the power $k$-means is smoother than that of the kernel $k$-means, so the solution is less likely to fall into local minima and is more robust. When $s \to -\infty
$, $f_s(\mathbf{\Theta})$ converges to the objective of kernel $k$-means since $M_s(\boldsymbol{y})\to \min\{y_1,y_2,\dots,y_k\}$, which implies that the minimizer $f_s(\mathbf{\Theta})$ gradually converges to the minimizer of the $k$-means.

\subsection{Random Fourier Features}
RFF was first introduced in \cite{rahimi2007random}, which maps the data into a low-dimensional space such that the inner product of the mapped data points approximates the kernel similarity between them. Specifically, for shift-invariant kernel $k(\boldsymbol{x},\boldsymbol{y})$ such that it can be expressed as $K(\boldsymbol{x}-\boldsymbol{y}): \mathbb{R}^d \to \mathbb{R}$, function $p: \mathbb{R}^d \to \mathbb{R}$ such that $p(\boldsymbol{\omega}) =  \frac{1}{2\pi} \int_{\mathbb{R}^d} K(\boldsymbol{x}) e^{-i\boldsymbol{\omega}^{\top} \boldsymbol{x}} \, d\boldsymbol{x}$, the RFF mapping is defined as:
\begin{equation}\label{def:RFF}
\tilde\phi(x) := \sqrt{\frac{1}{D}} 
\begin{pmatrix}
    \sin(\boldsymbol{\omega}_1^{\top} \boldsymbol{x}) \\
    \cos(\boldsymbol{\omega}_1^{\top} \boldsymbol{x}) \\
    \vdots \\
    \sin(\boldsymbol{\omega}_D^{\top} \boldsymbol{x}) \\
    \cos(\boldsymbol{\omega}_D^{\top} \boldsymbol{x})
\end{pmatrix},
\end{equation}
where frequency vectors $\boldsymbol{\omega}_1, \boldsymbol{\omega}_2, \dots , \boldsymbol{\omega}_D \in \mathbb{R}^d$ are i.i.d. sampled from distribution with density $p$.

Theoretically, \citet{rahimi2007random} have proven that $\mathbb{E}[\langle \tilde\phi(\boldsymbol{x}), \tilde\phi(\boldsymbol{y}) \rangle] = \frac{1}{D} \sum_{i=1}^{D} \mathbb{E}[\cos(\boldsymbol{\omega}_i, \boldsymbol{x}-\boldsymbol{y})] = K(\boldsymbol{x}-\boldsymbol{y}).$ Practically, compared to the Nyström method \cite{williams2000using}, the RFF approach directly generates low-dimensional random feature vectors, eliminating the need to store kernel matrices while bypassing eigendecomposition and subsample selection. The application of RFF significantly reduces both space and time complexity.

\section{RFF-based Kernel Power $K$-means}
\subsection{Problem Statement}

In RFF-KPKM, we aim to optimize the membership matrix as our primary objective rather than the cluster centroids. This is because we can directly use the membership matrix as input for algorithms like KKM or KPK to compute the approximation loss. However, the cluster centroids output by RFF-KPKM cannot since they cannot reflect the position of cluster centroids before mapping. To achieve the goal, let us review the update rule for cluster centroid $\boldsymbol{\theta}_{j}$ in KPKM:
\begin{equation}
    \boldsymbol{\theta}_j^{(m+1)} = \frac{\sum_{i=1}^n w_{ij}^{(m)} \phi(\boldsymbol{x}_i)}{\sum_{i=1}^n w_{ij}^{(m)}},
\end{equation}
where
\begin{equation}\label{eq:membership function}
    w_{ij}^{(m)} = \frac{\frac{1}{k} \left\| \phi(\boldsymbol{x}_i) - \theta_j^{(m)} \right\|^{2(s-1)}}{\left(\frac{1}{k} \sum_{l=1}^k \left\| \phi(\boldsymbol{x}_i) - \theta_l^{(m)} \right\|^{2s}\right)^{(1-1/s)}}.
\end{equation}
By the observation that $\boldsymbol{\theta}_j $ is formed as the convex combination of mapped data, we can reformulate KPKM as follows:
\begin{equation}\label{def:kernel power kmeans}
\begin{aligned}
     &\min_{\mathbf{W}\in\mathbb{R}^{n\times k}}f_s(\mathbf{W}) = \sum_{i=1}^nM_s\left(d_{i1}^2,\ldots,d_{ik}^2\right)\\
    &\ \ \ \text{s.t.} \ \ \forall j\in\{1,\dots,k\},\ \ \sum_{i=1}^{n}W_{ij}=1,\\&
    \ \ \ \ \ \ \ \ \ \  d_{ij} = \|\phi(\boldsymbol{x}_i)-\boldsymbol{\theta}_j\|,\ \  \boldsymbol{\theta}_{j}=\mathbf{\Phi} \mathbf{W}^{(j)},
\end{aligned}
\end{equation}
where $\mathbf{\Phi}$ is the mapped data matrix and the $i$-th column of $\mathbf{\Phi}$ is $\phi(\boldsymbol{x}_i)$. Such reformulation is quite reasonable according to Theorem 1 in \cite{paul2022implicit}: the minimizer of $f_s(\mathbf{\Theta})$ lies in a convex hull in the compact Cartesian product $\mathrm{conv}(\phi(C))^k$. We refer to $\mathbf{W}$ as the membership matrix, as it determines the weight of data points in influencing the position of cluster centroids. To simplify notation, we denote $M_s(\|\phi(\boldsymbol{x}_i)-\mathbf{\Phi} \mathbf{W}^{(1)}\|^2,\ldots,\|\phi(\boldsymbol{x}_i)-\mathbf{\Phi} \mathbf{W}^{(k)}\|^2)$ as $M_{s}(\phi(\boldsymbol{x}_i),\mathbf{W})$. 

It is crucial to note that we cannot directly access the explicit form of accurate cluster centroids during iterations. KPKM solves the problem by using the kernel trick because the algorithm's update process only requires computing point-to-cluster distances. However, the computation of the kernel matrix leads to a huge computational cost.
Unlike KPKM, we employ RFF to calculate low-dimensional embeddings of the data as a surrogate for their high-dimensional representations $\phi(\boldsymbol{x})$ in the Hilbert space. This allows us to explicitly compute cluster centroids during iterations while significantly reducing the computational overhead associated with kernel matrix operations. 
By substituting the origin implicit mapping $\phi$ in Eq. (\ref{def:kernel power kmeans}) with RFF mapping $\tilde{\phi}:\mathbb{R}^{d}\to\mathbb{R}^{D}$ defined in Eq. (\ref{def:RFF}) where $D\ll d$, we obtain the objective function $\tilde{f}_{s}(\mathbf{W})$ of proposed RFF-KPKM. The optimization procedure for RFF-KPKM mirrors that of standard PKM, with the distinction that feature vectors $\boldsymbol{x}$ are replaced by our constructed RFF mapping $\tilde\phi(\boldsymbol{x})$. The complete pseudocode for the optimization algorithm is provided in the appendix.

To evaluate the effect upon precision and consistency when substituting $\phi(\boldsymbol{x})$ with $\tilde{\phi}(\boldsymbol{x})$, we establish an excess risk bound, strong consistency, and a relative error bound theoretically. The full result is stated in the following subsection.


\subsection{Theoretical Analysis}
\paragraph{Excess Risk Bound}
In this section, we obtain the excess risk bound of the proposed RFF-KPKM. Here, we assume that our analysis is confined to the Gaussian kernel scenario; however, our theorems also hold for all positive definite translation-invariant kernels that take strictly positive values (see the proof in the Appendix). 
We now define the expected version of the KPKM objective function as:
\begin{equation}
    \mathcal{L}_s(\mathbf{W},\rho)=\int M_{s}(\phi(\boldsymbol{x}),\mathbf{W})d\rho(\boldsymbol{x}).
\end{equation}
Given a training set $S_n=\{\boldsymbol{x}_i\}_{i=1}^n$, the empirical version of the KPKM and RFF-KPKM can be written as:
\begin{equation}
        \mathcal{L}_s(\mathbf{W},\rho_n) =\frac{1}{n}f_s(\mathbf{W}),
\end{equation}

\begin{equation}
\tilde{\mathcal{L}}_s(\mathbf{W},\rho_n)=\frac{1}{n}\tilde{f}_s(\mathbf{W}).
\end{equation}
The empirical risk minimizer (ERM) is defined as:
\begin{equation}
    \mathbf{W}_{n,s} := \arg\min_{\mathbf{W} \in \mathcal{A}} \mathcal{L}_s(\mathbf{W},\rho_n),
\end{equation}
\begin{equation}
    \tilde{\mathbf{W}}_{n,s} := \arg\min_{\mathbf{W} \in \mathcal{A}} \tilde{\mathcal{L}}_s(\mathbf{W},\rho_n),
\end{equation}
where $\mathcal{A}=\{\mathbf{X}\in\mathbb{R}^{n\times k}:\sum_{i=1}^{n}\mathbf{X}_{ij}=1, \forall j\in\{1,...,k\}\}.$
We can now define the excess risk of RFF-KPKM as:
\begin{equation}
    \mathcal{E}(\tilde{\mathbf{W}}_{n,s}) := \mathbb{E}[\mathcal{L}_s(\tilde{\mathbf{W}}_{n,s}, \rho)] - \mathcal{L}_s^*(\rho).
\end{equation}

\begin{theorem}\label{theorem: theorem1}
When the dimension of RFF $D=\Omega(\frac{nd\log(n/k\delta)}{k})$, with probability at least $1-\delta$ we have
\begin{equation}
    \mathcal{E}(\tilde{\mathbf{W}}_{n,s}) = \mathbb{E}[\mathcal{L}_s(\tilde{\mathbf{W}}_{n,s}, \rho)] - \mathcal{L}_s^*(\rho)\leq \mathcal{O}\left(\sqrt\frac{k^{3-2/s}}{n}\right),
\end{equation}
where $\mathcal{L}_s^*(\rho) = \inf_{\mathbf{W}} \mathcal{L}_s(\mathbf{W}, \rho)$. 
\end{theorem}

\noindent\textbf{Remark: }
    To our knowledge, Theorem \ref{theorem: theorem1} presents the first generalization error bounds that quantify the approximation efficacy of RFF in kernel clustering tasks. Although \cite{cheng2023relative} established a $(1+\varepsilon)$ relative error bound for RFF-based $k$-means, their results lack the explicit generalization performance of RFF. Separately, \cite{yin2022randomized} derived excess risk bounds for $k$-means based on randomized sketches. However, such sketch methods as random projection, ROS, and Nyström require computing the full or partial kernel matrix, thus failing to circumvent its computational burden. 

\paragraph{Strong Consistency}
In this section, we analyze the strong consistency of the ERM of RFF-KPKM: $\tilde{\mathbf{W}}_{n,s}$. First, we denote the expected version of KKM as follows:
\begin{equation}
    \Psi(\mathbf{W}, \rho) = \int \min_{1\leq j\leq k}\|\phi(\boldsymbol{x}_{i})-\mathbf{\Phi}\mathbf{W}^{(j)}\|^2d\rho(\boldsymbol{x}),
\end{equation}
\begin{equation}
    \tilde\Psi(\mathbf{W}, \rho) = \int \min_{1\leq j\leq k}\|\tilde\phi(\boldsymbol{x}_{i})-\tilde{\mathbf{\Phi}}\mathbf{W}^{(j)}\|^2d\rho(\boldsymbol{x}).
\end{equation}
The minimizer of $\Psi(\mathbf{W}, \rho)$ is defined as 
\begin{equation}
    \mathbf{W}^* := \arg\min_{\mathbf{W} \in \mathcal{A}} \Psi(\mathbf{W}, \rho),
\end{equation}
\begin{equation}
    \tilde{\mathbf{W}}^* := \arg\min_{\mathbf{W} \in \mathcal{A}} \tilde\Psi(\mathbf{W}, \rho).
\end{equation}
Establishing consistency equals showing that $\tilde{\mathbf{W}}_{n,s} \xrightarrow{\text{a.s.}} \mathbf{W}^*$ when $n\to \infty$, the dimensionality of RFF  $D\to \infty$ and $s\to -\infty$. To address this issue, we introduce the following standard assumption, which allows us to study the convergence of variables by analyzing the limit of the objective function sequence.\\
\textbf{Assumption 1. }Let $B(\mathbf{W}, r)$ denote the the ball with radius $r$ around $\mathbf{W}$. For any $r > 0$, there exists $ \varepsilon > 0$ such that for all $ \mathbf{W} \in \mathcal{A} \setminus B(\mathbf{W}^*, r)$, we have $\Psi(\mathbf{W}, \rho) > \Psi(\mathbf{W}^*, \rho) + \varepsilon$.
\begin{theorem}\label{theorem:consistency}
    Under Assumption 1, $\tilde{\mathbf{W}}_{n,s} \xrightarrow{\text{a.s.}} \mathbf{W}^*$ with any constant probability when $n\to \infty$, RFF dimensionality $D\to \infty$ and $s\to -\infty$. 
\end{theorem}
\noindent\textbf{Remark: }Notably, the assumption is standard for the analysis of the consistency of clustering \cite{paul2022implicit,pollard1981strong,chakraborty2020detecting}. 
According to Assumption 1, to build consistency is actually to prove that $\Psi(\tilde{\mathbf{W}}_{n,s}, \rho)\xrightarrow{\text{a.s.}}\Psi(\mathbf{W}^*, \rho)$ when $n, D\to \infty$ and $s\to -\infty$. However, Theorem \ref{theorem:consistency} does not specify the appropriate dimension $D$ for RFF in practical implementations. The following theorem addresses this gap by providing a theoretical justification for the choice of $D$.

\paragraph{Relative Error} In this section, we study the relative error bound deriving from using ERM of RFF-KPKM with finite RFF as the input of expected KKM. We obtain that $\mathrm{poly} (\varepsilon^{-1}\log{k})$ RFF are enough to obtain the $(1+\varepsilon)$ error bound. The following theorem shows the main result.
\begin{theorem}\label{theorem:approx consistency}
    Under Assumption 1, if RFF dimensionality $\Omega(( \log^3 (\frac{k}{\delta}) + \log^3 (\frac{1}{\varepsilon}) + 2^8)/\varepsilon^2)$, then with probability at least $1-\delta$ we have
\begin{equation}
\lim_{n,-s\to\infty}\Psi(\tilde{\mathbf{W}}_{n,s},\rho)
    \leq\frac{1+\varepsilon}{1-\varepsilon}\cdot\Psi(\mathbf{W}^*, \rho).
\end{equation}
\end{theorem}
\noindent\textbf{Remark: } Notably, by rescaling $\varepsilon'=2\varepsilon/(1-\varepsilon)$, we have $\lim_{n,s\to\infty}\Psi(\tilde{\mathbf{W}}_{n,s},\rho)
    \leq(1+\varepsilon')\cdot\Psi(\mathbf{W}^*, \rho)$, which shows the method achieves the $(1+\varepsilon')$ relative error. By Theorem \ref{theorem:approx consistency}, we recommend setting the RFF dimension $D$ to $\lceil4\log^3(2k)\rceil$ to ensure the accuracy, which is obtained by setting $\varepsilon$ and $\delta$ to $0.5$ and preserving the main term in the optimal $D$. It should be noted that the optimal $D=\mathrm{poly}(\varepsilon^{-1}\log k)$ is independent of $n$ and $d$ and sublinear to $k$.
    
\subsection{Complexity Analysis}
The computational complexity of RFF-KPKM primarily consists of the following main steps. For generating the RFF mapping, the complexity is $\mathcal{O}(ndD)$. For updating the cluster centroids $\mathbf{\Theta}$, the complexity is $\mathcal{O}(nkD)$. For computing the distance between data points and cluster centroids, the complexity is $\mathcal{O}(nkD)$. In sum, RFF-MKPKM consumes a time complexity of $\mathcal{O}(n(k+d)D)$ in each iteration, greatly reducing the time complexity $\mathcal{O}(n^2k)$ of KPKM\cite{paul2022implicit} when $n$ is large.

\section{Improved Possibilistic RFF-based Multiple Kernel Power $K$-means}
\subsection{Problem Statement}
To solve the noise problem in RFF-KPKM and extend it to multi-view settings, we combine the PCM and the kernel combination method with RFF-KPKM. 
As mentioned before, fuzzy membership functions such as E.q. (\ref{eq:membership function}) are not appropriate to reflect the belonging degrees since they are determined by the data point to all cluster centroids. To obtain a more appropriate membership, a natural thought is to redefine a new membership as variables to be optimized. PCM inspire us to substitute $d_{ij}=\|\tilde\phi(\boldsymbol{x}_{i})-\theta_{j}\|^2$ in RFF-KPKM by the following function:
 \begin{equation}\label{eq:new distance}
     \tilde{d}_{ij}=(u_{ij})^{m}\|\tilde\phi(\boldsymbol{x}_i)-\boldsymbol{\boldsymbol{\theta}}_{j}\|^2
     +(1-u_{ij})^{m}\eta_{j},
 \end{equation}
where $u_{ij}$ represents the possibility membership, $\eta_{j}$ is a regularization hyper-parameter and $m$ controls the fuzziness of the membership degrees. Notably, we can not remove the term $(1-u_{ij})^{m}\eta_{j}$ since it is crucial to avoid the trivial solution and produce the membership that only depends on a point to a cluster.
 Unlike PCM, our method obtains the possibilistic membership while preserving the fuzzy membership $w_{ij}$, which can be obtained naturally through calculating the partial derivative of the objective function. The final cluster centroids are formed as
\begin{equation}
\boldsymbol{\theta}_{j} = \frac{\sum_{i=1}^{n} w_{ij}\left(u_{ij}\right)^m \tilde\phi(x_i)}{\sum_{i=1}^{n} w_{ij}\left(u_{ij}\right)^m}.
\end{equation}
We define the new membership function $\gamma_{ij} = w_{ij}\left(u_{ij}\right)^m$, incorporating both possibilistic and fuzzy memberships. $\gamma_{ij}$ enhances robustness because $u_{ij}$ measures the absolute degree of typicality of a point in any cluster. Furthermore, the constraint  $\sum_{j=1}^{k} w_{ij} = 1$ within $\gamma_{ij}$ prevents identical clusters by ensuring a data point's high fuzzy membership in one cluster necessitates lower memberships in others.

Clustering in multi-view scenarios is becoming increasingly popular and important \cite{zhou2025dynamic,wangevaluate,wang2024view,liang2024mgksite,liang2025from}. To extend to multiple kernel settings, we subsitute $\tilde\phi(\boldsymbol{x}_i)$ by $\tilde{\phi}_{\alpha}(\boldsymbol{x}):= ( \sqrt{\alpha_1} \tilde{\phi}_1(\boldsymbol{x})^{\top}, \ldots, \sqrt{\alpha_L} \tilde{\phi}_L(\boldsymbol{x})^{\top} )^{\top}$, where $\tilde\phi_l(\boldsymbol{x}_i)$ denotes the RFF for $\boldsymbol{x}_{i}$ in the $l$-th view and $\boldsymbol{\alpha} = (\alpha_1, \ldots, \alpha_L)$ act as weights. Finally, by re-parameterizing $\boldsymbol{\theta}_j = ( \frac{1}{\sqrt{\alpha_1}} \theta_{j,1}^\top, \ldots, \frac{1}{\sqrt{\alpha_L}} \theta_{j,L}^\top )^\top$, the objective function of IP-RFF-MKPKM is defined as follows:
\begin{equation}\label{eq:IP-RFF-MKPKM}
\begin{aligned}&\min_{\boldsymbol{\alpha},\mathbf\Theta,\mathbf{U}}f_s(\boldsymbol{\alpha},\mathbf\Theta,\mathbf{U}) =  \sum_{i=1}^{n}\mathcal{M}_s\left(\boldsymbol{\alpha},\boldsymbol{d}_{i}\right)+\lambda\sum_{l=1}^{L}\alpha_l\log\alpha_l\\\ \ \  &
\mathrm{s.t.} \ \ \mathcal{M}_s\left(\boldsymbol{\alpha},\boldsymbol{d}_{i}\right) = M_s\left(\sum_{l=1}^{L}\alpha_l\tilde{d}_{i1,l},\dots,\sum_{l=1}^{L}\alpha_l\tilde{d}_{ik,l}\right),\\&
\ \ \ \ \ \ \ \tilde{d}_{ij,l}=(u_{ij})^{m}\|\tilde\phi_l(\boldsymbol{x}_i)-\boldsymbol{\boldsymbol{\theta}}_{j,l}\|^2 + (1-u_{ij})^{m}\eta_{j,l},\\&
\ \ \ \ \ \ \ \forall j\in\{1,\dots,k\},\ \ \sum_{l=1}^{L}\alpha_l=1.
\end{aligned}
\end{equation}


\subsection{Optimization}
We adopt the MM algorithm to minimize Eq. (\ref{eq:IP-RFF-MKPKM}). First, we utilize the concave property of power means to obtain the surrogate function $g_s(\boldsymbol{\alpha},\mathbf\Theta,\mathbf{U})$ of $f_s(\boldsymbol{\alpha},\mathbf\Theta,\mathbf{U})$. Then we developed an alternating optimization algorithm to optimize $g_s$, which simultaneously ensures the reduction of $f_s$ . Through the tangent plane inequality described in \cite{paul2022implicit}, we can obtain that
\begin{equation}
\begin{aligned}\label{eq:MM update}
&f_{s}(\boldsymbol{\alpha},\mathbf\Theta,\mathbf{U})  \leq f_{s}(\boldsymbol{\alpha}^{(t)},\mathbf\Theta^{(t)},\mathbf{U}^{(t)}) \\
&- \sum_{i=1}^n \sum_{j=1}^k w_{ij}^{(t)} \sum_{l=1}^{L}\alpha_l^{(t)}\tilde{d}_{ij,l}^{(t)} 
-\lambda\sum_{l=1}^{L}\alpha_l^{(t)}\log\alpha_l^{(t)}\\ 
&+ 
\sum_{i=1}^n \sum_{j=1}^k w_{ij}^{(t)}\sum_{l=1}^{L}\alpha_l\tilde{d}_{ij,l}
+\lambda\sum_{l=1}^{L}\alpha_l\log\alpha_l.
\end{aligned}
\end{equation}
Here $w_{ij}$ are obtained from the partial derivatives of $M_s$:

\begin{equation}
            w_{ij}^{(t)} = \frac{\frac{1}{k}(\sum_{l=1}^{L}\alpha_l^{(t)}\tilde{d}_{ij,l}^{(t)})^{(s-1)}}
{(\frac{1}{k}\sum_{c=1}^k (\sum_{l=1}^{L}\alpha_l^{(t)}\tilde{d}_{ic,l}^{(t)})^{s})^{(1-1/s)}},
\end{equation}
where
\begin{equation}
    \tilde{d}_{ij,l}^{(t)}=(u_{ij}^{(t)})^{m}\|\tilde\phi_l(\boldsymbol{x}_i)-\boldsymbol{\boldsymbol{\theta}}_{j,l}^{(t)}\|^2
     +(1-u_{ij}^{(t)})^{m}\eta_{j,l}.
\end{equation}
We now define the entire right-hand side of Inequality (20) as $g_s(\boldsymbol{\alpha},\mathbf\Theta,\mathbf{U})$. We readily observe that the surrogate function $g_s$ satisfies the conditions required by the MM algorithm, so we need to find $(\boldsymbol{\alpha}^{(t+1)},\mathbf\Theta^{(t+1)},\mathbf{U}^{(t+1)})$ such that $g_{s}(\boldsymbol{\alpha}^{(t+1)},\mathbf\Theta^{(t+1)},\mathbf{U}^{(t+1)})\leq g_{s}(\boldsymbol{\alpha}^{(t)},\mathbf\Theta^{(t)},\mathbf{U}^{(t)})$. Ignoring the constant terms, optimizing $g_{s}(\boldsymbol{\alpha},\mathbf\Theta,\mathbf{U})$ is equivalent to solving the following optimization problem when the sum of $\alpha_{l}$ equals $1$:
\begin{equation}\label{eq:J_s}
\min_{\boldsymbol{\alpha},\mathbf\Theta,\mathbf{U}}\sum_{i=1}^n \sum_{j=1}^k w_{ij}^{(t)}\sum_{l=1}^{L}\alpha_l\tilde{d}_{ij,l}+\lambda\sum_{l=1}^{L}\alpha_l\log\alpha_l.
\end{equation}

 To address the optimization problem in Eq. (\ref{eq:J_s}), we propose a three-step alternating iterative algorithm. When optimizing a variable, the other variables are fixed to their previous iteration values. The complete steps are as follows.
 
\noindent\textbf{Step1:} Update $\mathbf{U}$. By substituting $\boldsymbol{\alpha}=\boldsymbol{\alpha}^{(t)}$, $\mathbf{\Theta}=\mathbf{\Theta}^{(t)}$ into Eq. (\ref{eq:J_s}) and removing terms unrelated to $\mathbf{U}$, we obtain the following optimization problem:
\begin{equation}\label{eq:step1}
\begin{aligned}
    \min_{\mathbf{U}}&\sum_{i=1}^n \sum_{j=1}^k w_{ij}^{(t)}
    \left( (u_{ij})^{m}\sum_{l=1}^{L}\alpha_l^{(t)}\|\tilde\phi_l(\boldsymbol{x}_i)-\boldsymbol{\boldsymbol{\theta}}_{j,l}^{(t)}\|^2 \right.
     \\&   +(1-u_{ij})^{m} \left.\sum_{l=1}^{L}\alpha_l^{(t)}\eta_{j,l} \right). 
\end{aligned}
\end{equation}
Setting the derivative of Eq. (\ref{eq:step1}) with respect to 
$u_{ij}$ to zero, we obtain the update rule of $u_{ij}$ as follows:
\begin{equation}
    u_{ij}^{(t+1)} = \frac{1}{1 + \left(\frac{\sum_{l=1}^{L}\alpha_l^{(t)}\|\tilde\phi_l(\boldsymbol{x}_i)-\boldsymbol{\boldsymbol{\theta}}_{j,l}^{(t)}\|^2}{\sum_{l=1}^{L}\alpha_l^{(t)}\eta_{j,l}}\right)^{\frac{1}{m - 1}}}.
\end{equation}

\noindent\textbf{Step2:} Update $\mathbf{\Theta}$. By substituting $\boldsymbol{\alpha}=\boldsymbol{\alpha}^{(t)}$, $\mathbf{U}=\mathbf{U}^{(t+1)}$ into Eq. (\ref{eq:J_s}) and removing terms unrelated to $\mathbf{\Theta}$, we obtain the following optimization problem:
\begin{equation}\label{eq:step2}
    \min_{\mathbf{\Theta}}\sum_{i=1}^n \sum_{j=1}^k w_{ij}^{(t)}
    \left( \left(u_{ij}^{(t+1)}\right)^{m}\sum_{l=1}^{L}\alpha_l^{(t)}\|\tilde\phi_l(\boldsymbol{x}_i)-\boldsymbol{\boldsymbol{\theta}}_{j,l}\|^2
     \right).
\end{equation}
Setting the derivative of Eq. (\ref{eq:step2}) with respect to
$\boldsymbol{\theta}_{j,l}$ to zero, we obtain the update rule of $\boldsymbol{\theta}_{j,l}$ as follows:
\begin{equation}
    \boldsymbol{\theta}_{j,l}^{(t+1)} = \frac{\sum_{i=1}^{n} w_{ij}^{(t)}\left(u_{ij}^{(t+1)}\right)^m \tilde\phi_l(\boldsymbol{x}_i)}{\sum_{i=1}^{n} w_{ij}^{(t)}\left(u_{ij}^{(t+1)}\right)^m}.
\end{equation}

\noindent\textbf{Step3:} Update $\boldsymbol{\alpha}$. By substituting $\mathbf{\Theta}=\mathbf{\Theta}^{(t+1)}$, $\mathbf{U}=\mathbf{U}^{(t+1)}$ into Eq. (\ref{eq:J_s}) and removing terms unrelated to $\mathbf{\Theta}$, we obtain the following optimization problem:
\begin{equation}\label{eq:step3}
\min_{\boldsymbol{\alpha}}\sum_{l=1}^{L}\alpha_l\sum_{i=1}^n \sum_{j=1}^k w_{ij}^{(t)}\tilde{d}_{ij,l}^{(t+1)}+\lambda\sum_{l=1}^{L}\alpha_l\log\alpha_l.
\end{equation}
To minimize Eq. (\ref{eq:step3}) in $\boldsymbol{\alpha}$,  we consider the
 Lagrangian
 \begin{equation}
 \begin{aligned}
    \mathcal{L}(\boldsymbol{\alpha},\beta)=&\sum_{l=1}^{L}\alpha_l\sum_{i=1}^n \sum_{j=1}^k w_{ij}^{(t)}\tilde{d}_{ij,l}^{(t+1)}\\&
    +\lambda\sum_{l=1}^{L}\alpha_l\log\alpha_l-\beta(\sum_{l=1}^{L}\alpha_l-1).
 \end{aligned}
 \end{equation}
 By setting $\frac{\partial \mathcal{L}}{\partial \alpha_l} = 0$ and combining with $\sum_{l=1}^{L}\alpha_l=1$,  we obtain the update rule of $\alpha_l$ as follows:
\begin{equation}
    \alpha_l^{(t+1)} =\frac{\exp\left(-\frac{1}{\lambda} \sum_{i=1}^{n} \sum_{j=1}^{k} w_{ij}^{(t)} \tilde{d}_{ij,l}^{(t+1)}\right)}{\sum_{l=1}^{L} \exp\left(-\frac{1}{\lambda} \sum_{i=1}^{n} \sum_{j=1}^{k} w_{ij}^{(t)} \tilde{d}_{ij,l}^{(t+1)}\right)}.
\end{equation}

\subsection{Complexity Analysis}
The computational complexity of IP-RFF-MKPKM primarily consists of the following main steps. For generating the RFF mapping of the $L$ kernels, the complexity is $\mathcal{O}(ndDL)$. For updating the membership matrix $\mathbf{U}$, the complexity is $\mathcal{O}(nkDL)$. For updating the cluster centroids $\mathbf{\Theta}$, the complexity is $\mathcal{O}(nkDL)$. For computing the distance between data points and cluster centroids, the complexity is $\mathcal{O}(nkDL)$. In sum, IP-RFF-MKPKM consumes a time complexity of $\mathcal{O}(n(k+d)DL)$ in each iteration, greatly reducing the time complexity $\mathcal{O}(n^2kL)$ of MKPKM.

\section{Experiment}
In this section, we mainly discuss the experimental results of the multi-view clustering method IP-RFF-MKPKM due to space constraints. Single-view experiments and parameter settings are provided in the Appendix.
\subsection{Experimental Settings}

\paragraph{Datasets} To evaluate the effectiveness of the proposed RFF-MKPKM, we employ six single-view datasets derived from the first view of six multi-view datasets: Yale \cite{cai2005using}, Caltech101-20 \cite{li2015large}, 100leaves\cite{wang2019gmc}, CIFAR10 \cite{krizhevsky2009learning}, YTF10 and YTF20 \cite{wolf2011face}. To evaluate the effectiveness of the proposed IP-RFF-MKPKM, we use seven widely used multi-view datasets: LGG \cite{cancer2015comprehensive}, Caltech101-7 \cite{li2015large}, HW2 \cite{Kevin2013}, NUS-WIDE-SCENE \cite{chua2009nus}, NUS-WIDE-OBJECT \cite{chua2009nus} and CIFAR10 \cite{krizhevsky2009learning}. Comprehensive details regarding these datasets are accessible in the appendix.

\paragraph{Compared Methods} To contextualize the innovation of RFF-KPKM, we compare the proposed RFF-KPKM against four pivotal baselines: kernel $k$-means \cite{girolami2002mercer}, spectral clustering \cite{von2007tutorial}, power $k$-means \cite{xu2019power}, and kernel power $k$-means \cite{paul2022implicit}. The proposed IP-RFF-MKPKM is evaluated against eight contemporary multi-view clustering (MVC) methods, including: MVASM \cite{han2020multi}, MKPKM \cite{paul2022implicit}, OMVFC-LICAG \cite{zhang2024latent}, LKRGDF \cite{chen2024multiple}, SVD-SMKKM \cite{liang2024consistency}, INMKC \cite{feng2025incremental}, SMKKM-UGF \cite{yang2025smooth}, MVHBG \cite{zhao2025multi}. 


\paragraph{Evaluation} In all experiments, we use RFF of dimension $\lceil4\log^3(2k)\rceil$ according to Theorem \ref{theorem:approx consistency}. For statistical robustness, every trial was repeated 20 times with average performance metrics. Clustering quality was evaluated through three standard criteria: Accuracy (ACC), Normalized Mutual Information (NMI), and Purity. The computational environment comprised an Intel Core i9-10900X CPU, 64GB RAM, and MATLAB 2020b (64-bit).


\begin{table*}[ht]
\centering
\scalebox{0.81}{\begin{tabular}{ccccccccccc}
\hline
\multirow{2}{*}{Dataset}                        & \multirow{2}{*}{Metric }       & \multirow{2}{*}{MVASM }        & \multirow{2}{*}{MKPKM}      & \multirow{-1}{*}{OMVFC-} & \multirow{2}{*}{LKRGDF} &\multirow{2}{*}{SVD-SMKKM}   & \multirow{2}{*}{SMKKM-UGF}           & \multirow{2}{*}{MVHBG}   & \multirow{2}{*}{INMKC}     & \multirow{2}{*}{Proposed}  \\ 
& & & &\multirow{-1}{*}{LICAG} &  &&&&
\\
\hline
                               & ACC          &59.25	&58.05 &57.30	&43.45	&69.29		&50.94	&61.42		&71.16 &\textbf{80.90}
                               \\
                               & NMI          &33.02	&29.22 &27.33	&0.07		&36.61&18.30	&22.90		&37.11	 &\textbf{45.12}
                              \\
                            \multirow{-3}{*}{LGG}  & Purity       &51.23 &62.17	&59.93	&50.19	&69.29		&55.81	&61.80		&71.16 &\textbf{80.90}
                        \\ \hline
                               & ACC           &61.29	&47.69 &41.99	&52.17	&51.42		&44.52	&52.24		&36.50 &\textbf{69.61}
                                 \\
                               & NMI        &50.54	&45.74 &39.58	&36.65	&36.37		&50.48	&43.88		&28.85 &\textbf{51.63}
                                 \\
                        \multirow{-3}{*}{Caltech101-7}  & Purity &82.89	&83.24 &81.61	&84.33	&83.38		&83.72	&84.80	&80.94 &\textbf{85.89}
                                \\ \hline
                               & ACC          &73.23	&86.54	&72.75	&57.00 	&62.5	&78.57	&62.6	&72.35	&\textbf{90.15}
                       \\
                               & NMI           &73.77	&75.06 &70.25	&66.69	&56.57		&81.01	&69.28		&68.45 &\textbf{81.56}
                           \\
\multirow{-3}{*}{HW2}      & Purity        &65.03	&85.12 &74.70	&64.45	&65.90		&81.37	&67.50		&74.80 &\textbf{90.15} 
                        \\ \hline
                               & ACC           &34.93	&44.51 &38.60	&39.10	&43.42		&37.21	&\textbf{58.42}		&33.65 &54.19
                        \\
                               & NMI           &43.56	&61.11 &47.71	&39.99	&43.87		&55.36	&60.64	&40.25	&\textbf{61.61}
                        \\
                            \multirow{-3}{*}{Caltech101-20}         & Purity        &52.27	&76.15 &68.11	&65.21	&67.73		&73.31	&74.73	&64.63 &\textbf{76.82}
                           \\ \hline
                               & ACC         &11.22 &11.61	&9.26	&14.68	&17.12		&11.53	&20.19	&9.74	&\textbf{22.08}
                        \\
 \multirow{-2}{*}{ NUS-WIDE-}                      & NMI        &8.81	&8.70 &8.34	&9.64	&8.28		&9.43	&8.52	&7.11	&\textbf{10.00}  
                        \\
\multirow{-2}{*}{SCENE}     & Purity        &17.15	&28.89 &30.60	&33.11	&33.80	&\textbf{34.13}	&30.60	&32.94	&31.99
                           \\ \hline
                               & ACC          &13.47 &- &11.17	&-	&14.10	&-	&-	&12.26 &\textbf{21.12}
                        \\
\multirow{-2}{*}{NUS-WIDE-}                              & NMI          &12.40	&- &8.13	&-	&10.50 		&-	&-	&9.27	&\textbf{13.53}
                          \\
\multirow{-2}{*}{OBJECT}   & Purity       &9.80 &-	&21.02 	&-	&22.54 	 	&-	&-		&22.01 &\textbf{24.30}
                        \\ \hline
                            & ACC          &- &- &-	&-	&25.46	&-	&-	&20.61 &\textbf{27.88}
                        \\
  & NMI          &-	&- &-	&-	&12.32		&-	&-	&7.63	&\textbf{14.91}
                          \\
\multirow{-3}{*}{CIFAR10}   & Purity       &- &-	&- 	&-	&25.78	 	&-	&-		&21.88 &\textbf{28.56}\\ \hline
\end{tabular}}
\caption{ Clustering performance comparison of IP-RFF-MKPKM with eight baselines on seven benchmark datasets concerning clustering accuracy (ACC), normalized mutual information (NMI), and Purity. Maximum values are highlighted in boldface; the hyphen symbol (-) designates execution termination due to memory overallocation.}
    \label{tab:result comparison}
\end{table*}

\begin{table*}
    \centering
\scalebox{0.85}{\begin{tabular}{ccccccccc}
        \toprule
        \multirow{2}{*}{Metrics} & \multirow{2}{*}{Method} & \multicolumn{7}{c}{Datasets} \\ \cmidrule{3-9}
        & &LGG & Caltech101-7 & HW2 & Caltech101-20 & NUS-WIDE-SCENE & NUS-WIDE-OBJECT &CIFAR10   \\
        \midrule
        \multirow{2}{*}{ACC} & RFF-MKPKM & 62.55 & 61.67 & 83.00 & 53.39 & 12.41 &13.95 &26.08 \\
            & IP-RFF-MKPKM & \textbf{80.90} & \textbf{69.61} & \textbf{90.15} & \textbf{54.19} & \textbf{22.08}  & \textbf{21.12} & \textbf{27.88}\\
        \bottomrule
    \end{tabular}}
    \caption{ACC comparison with and without the possibilistic method.}
    \label{Comparison of Results With and Without the Possibilistic Method}
\end{table*}

\subsection{Experimental Results}
Table \ref{tab:result comparison} shows the clustering outcomes of IP-RFF-MKPKM on the seven benchmark datasets, and Table \ref{Comparison of Results With and Without the Possibilistic Method} displays the results with and without the possibilistic method. Furthermore, Fig. \ref{running} compares the running time of IP-RFF-MKPKM with other state-of-the-art methods. Analysis of the experimental findings yields the following conclusions:
\begin{enumerate}
    \item Our proposed IP-RFF-MKPKM algorithm generally outperforms existing MKC methods on most datasets. In terms of ACC, IP-RFF-MKPKM surpasses the second-best algorithm on the LGG, Caltech101-7, HW2, NUS-WIDE-OBJECT, and CIFAR10 datasets by margins of 9.74\%, 8.32\%, 3.61\%, 7.02\%, 2.42\% respectively. 
    \item When compared to the original MKPKM method, our approach generally achieves superior performance. On the datasets LGG, Caltech101-7, HW2, Caltech101-20, NUS-WIDE-SCENE, IP-RFF-MKPKM consistently outperforms MKPKM by 22.85\%, 21.92\%, 3.61\%, 9.68\%, and 10.47\% in terms of ACC. To demonstrate the effectiveness of the proposed possibilistic method, we conducted experiments on RFF-MKPKM by replacing $\tilde{d}_{ij}$ in IP-RFF-MKPKM with $d_{ij}$. As shown in Table \ref{Comparison of Results With and Without the Possibilistic Method}, the possibilistic method significantly improves the clustering performance of RFF-MKPKM.
    \item As shown in Fig. \ref{running}, IP-RFF-MKPKM can run on the CIFAR dataset containing $60000$ samples, while six methods cannot since memory limitation. When compared to the powerful and efficient MVC method, INMKC, IP-RFF-MKPKM surpasses it by 9.74\%, 33.11\%, 17.8\%, 20.54\%, 8.86\%, and 7.27\% in terms of ACC across all datasets. The results demonstrate the superior scalability and performance of our method.
\end{enumerate}

\subsection{Convergence and Sensitivity Analysis}
We conduct an experimental series elucidating IP-RFF-MKPKM's convergence characteristics. As depicted in Fig. \ref{Convergence study}, the objective function value monotonically decreases per iteration, achieving convergence at approximately 30 iterations on Caltech101-7. We also examine the sensitivity of the parameters $s_0$ and $\lambda$. As illustrated in Fig. \ref{Sensitive Study}, the proposed method is not significantly affected by $\lambda$ within [1, 1000], and performs better when $s_0$ is close to 15 on the Caltech101-7 dataset. The results show that the addition of an entropy regularization term facilitates learning smoother combination coefficients, and the adjustment of $s_0$ helps to find the optimal global solution of IP-RFF-MKPKM.

\begin{figure}[htbp]
	\centering
	\begin{subfigure}{0.49\linewidth}
		\centering
		\includegraphics[width=0.9\linewidth]{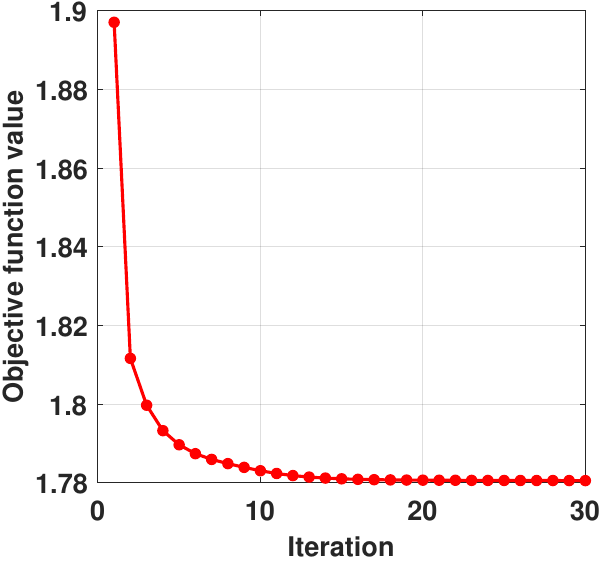}
		\caption{Convergence study}
		\label{Convergence study}
	\end{subfigure}
	\centering
	\begin{subfigure}{0.49\linewidth}
		\centering
		\includegraphics[width=0.9\linewidth]{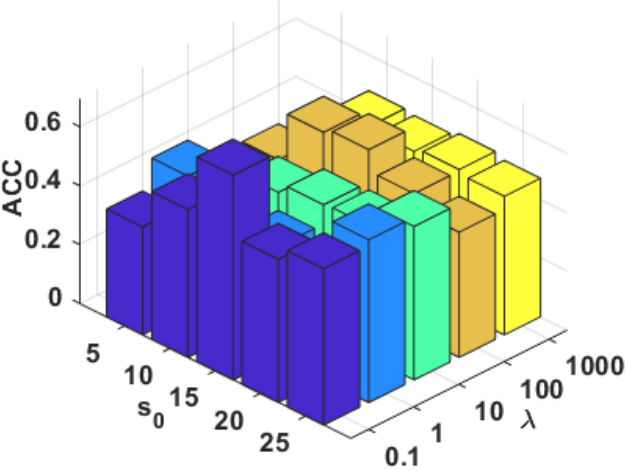}
		\caption{Study for $s_0$ and $\lambda$}
		\label{Sensitive Study}
	\end{subfigure}
	\caption{Convergence analysis and sensitivity analysis of $s_0$ and $\lambda$ of IP-RFF-MKPKM on Caltech101-7 dataset. Convergence and sensitivity studies on other benchmark datasets are given in the Appendix.}
\end{figure}

\subsection{Approximation Analysis}
To evaluate the impact of RFF dimensions on clustering accuracy, we tested the performance of RFF-KPKM across dimensions ranging from 5 to 100 on six datasets. As shown in Fig. \ref{acc_vs_dim}, below 40 RFF dimensions, dimensionality significantly impacts clustering accuracy (ACC); beyond 40 dimensions, ACC gains diminish and gradually approach stability. An RFF dimensionality of 20 to 40 represents the computational sweet spot, delivering $\geq85$\% of the maximum achievable accuracy for the datasets.


\begin{figure}[H]
    \centering
    \includegraphics[width=0.38\textwidth]{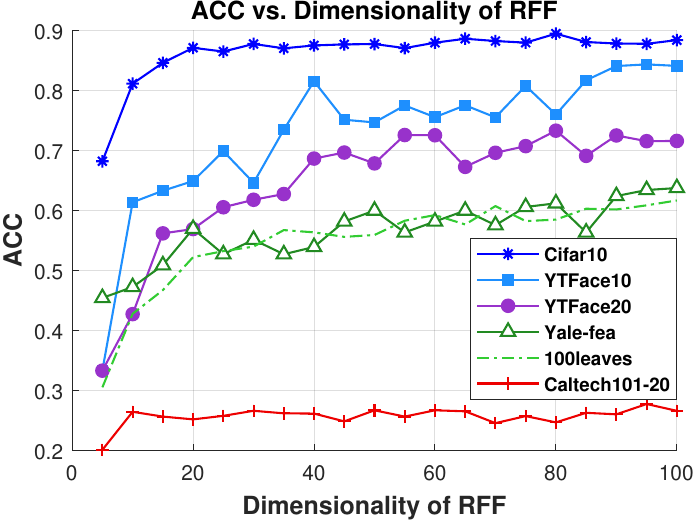}
    \caption{Impact of RFF dimension on RFF-KPKM clustering accuracy on six benchmark datasets. The  RFF dimension varies from 5 to 100.}
    \label{acc_vs_dim}
\end{figure}

\section{Conclusion}
This work utilizes random Fourier features (RFF) and a probabilistic method to address the time complexity and noise issue of KPKM. We demonstrate the effectiveness of RFF-KPKM by establishing its excess risk bound, strong consistency, and relative error, which establishes the first theoretical guarantees for the combination of RFF and KPKM. Building on this, we design IP-RFF-MKPKM by integrating possibilistic memberships and fuzzy memberships to redefine cluster affiliations, thereby eliminating outlier-induced centroid drift and extending scalability to multi-view settings. We conduct experiments on six benchmark single-view datasets and seven benchmark multi-view datasets and compare the proposed methods with state-of-the-art baseline methods. Experimental results demonstrate the effectiveness and efficiency of the proposed RFF-KPKM and IP-RFF-MKPKM.

\begin{figure}[t]
    \centering
    \includegraphics[width=0.47\textwidth]{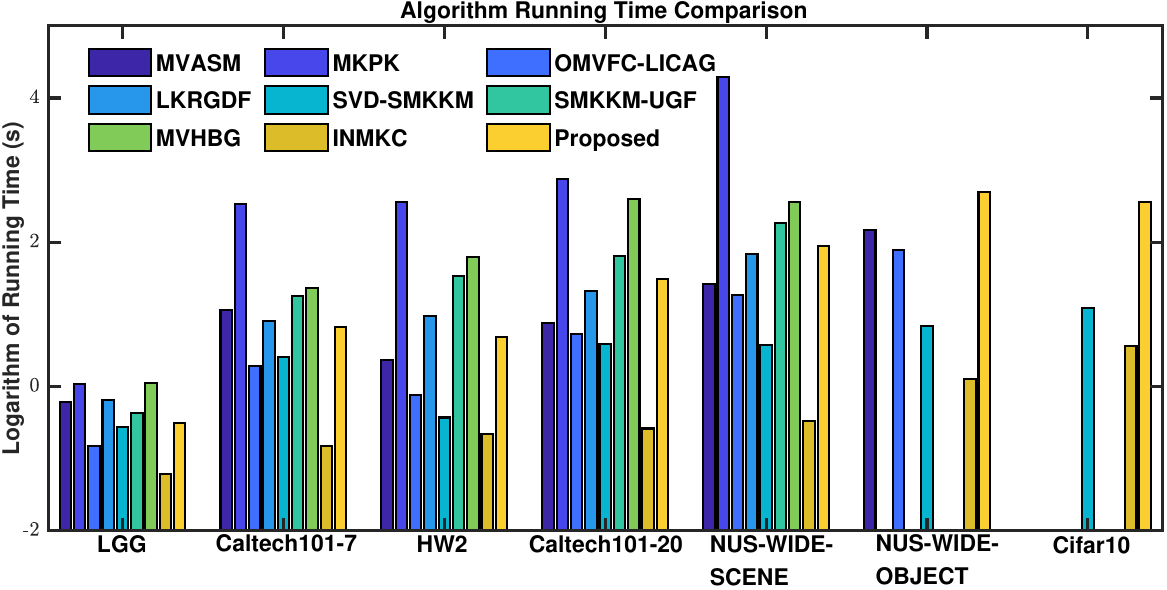}
    \caption{Logarithmic running time comparison of IP-RFF-MKPKM with eight benchmark methods on seven benchmark datasets. Bar absence denotes an out-of-memory runtime exception during method execution.}
    \label{running}
\end{figure}

\section{Acknowledgments}
This work is supported by the National Science Fund for Distinguished Young Scholars of China (No. 62325604), and the National Natural Science Foundation of China (No. 62441618, 62276271, 62506369).

\bibliography{main}

\onecolumn
\section{A. Proof of Theorems}
\subsection{A.1 Proof of Theorem 1}

To prove Theorem 1, we first note that
\begin{align*}
    & \mathbb{E}[\mathcal{L}_s(\tilde {\mathbf{W}}_{n,s}, \rho)] - {\mathcal{L}}_s^*(\rho)\\
    &=  \underbrace{\mathbb{E}[\mathcal{L}_s(\tilde{\mathbf{W}}_{n,s}, \rho) - \mathcal{L}_s(\tilde{\mathbf{W}}_{n,s}, \rho_n)}_{A_1}] + \underbrace{\mathbb{E}[\mathcal{L}_s(\tilde{ \mathbf{W}}_{n,s}, \rho_n) - \mathcal{L}_s( \mathbf{W}_{n,s}, \rho_n)]}_{A_2} \\
    & + \underbrace{\mathbb{E}[\mathcal{L}_s( \mathbf{W}_{n,s}, \rho_n) - \mathcal{L}_s( \mathbf{W}_{n,s}, \rho)]}_{A_3} + \underbrace{\mathbb{E}[\mathcal{L}_s( \mathbf{W}_{n,s}, \rho)] - \mathcal{L}_s^*(\rho)}_{A_4},
\end{align*}
and then we proceed to analyze the four error terms individually to establish their respective bounds.

\begin{lemma}\cite{paul2022implicit}\label{lemma1}
Assume there exists  M > 0  such that  $\| \boldsymbol{\theta} \| \leq M $ for all $ \boldsymbol{\theta} \in \mathrm{conv}(\phi(C))$.
With probability at least $1-\delta$,
\begin{equation}
    \begin{aligned}
        &\sup_{\mathbf{W} \in \mathcal{A}} \left| \mathcal{L}_s(\mathbf{W}, \rho) - \mathcal{L}_s(\mathbf{W}, \rho_n) \right|\leq 2\sqrt{2} Mk^{3/2 - 1/s} \left( 2\sqrt{\mathbb{E}_{x \sim P} \| \phi(x) \|^2 + M} \right) \frac{1}{\sqrt{n}}
        + 4M^2 \sqrt{\frac{\log(2/\delta)}{2n}}.
    \end{aligned}
\end{equation}
\end{lemma}

\paragraph{Remark: }
The assumption is easy to satisfy in the Gaussian kernel scenario since we have $\|\phi(\boldsymbol{x})\|=1$. Lemma \ref{lemma1} immediately implies that $A_1$ and $A_3\leq \mathcal{O}(\sqrt{{k^{3-2/s}}/{n}})$. For $A_4$, we just need to note that $A_4\leq 2\sup_{\mathbf{W} \in \mathcal{A}} \left| \mathcal{L}_s(\mathbf{W}, \rho) - \mathcal{L}_s(\mathbf{W}, \rho_n) \right|\leq \mathcal{O}(\sqrt{{k^{3-2/s}}/{n}})$ according to the proof of Theorem 5 in \cite{paul2022implicit}. 
Now, we focus on bounding $A_2$. We first establish the relative error between $k(\boldsymbol{x},\boldsymbol{y})$ and the inner product of Random Fourier Features. The following lemma extends Claim 1 in \cite{rahimi2007random}, which shows the additive error of RFF, to a multiplicative error scheme under the assumption of a Gaussian kernel.
\begin{lemma}\label{lemma:relative error of RFF}
    Assume there exists a constant $C > 0$ such that $\|\boldsymbol{x}\| < C$ for all $\boldsymbol{x}\in S_n$. For Gaussian kernel $k(\boldsymbol{x},\boldsymbol{y})=exp(\|\boldsymbol{x}-\boldsymbol{y}\|^2/(2\sigma^2))$ and RFF map $\tilde\phi:\mathbb{R}^{d}\to\mathbb{R}^{D}$ with $D=\Omega\left(\frac{d}{\varepsilon^2} \log \frac{1}{\varepsilon \sqrt{\delta}}\right)$, with probability at least $1-\delta$, we have
    \begin{equation}
        \sup_{\boldsymbol{x},\boldsymbol{y} \in S_n} |\langle\tilde\phi(\boldsymbol{x}),\tilde\phi(\boldsymbol{y})\rangle - k(\boldsymbol{x},\boldsymbol{y})|\leq \mathcal{O}(\varepsilon)k(\boldsymbol{x},\boldsymbol{y}).
    \end{equation}
\end{lemma}
\begin{proof}
According to Claim 1 in \cite{rahimi2007random}, with probability at least $1-\delta$, we have
    \begin{equation}
        \sup_{\boldsymbol{x},\boldsymbol{y} \in S_n} |\tilde\phi(\boldsymbol{x})^{\top} \tilde\phi(\boldsymbol{y}) - k(\boldsymbol{x},\boldsymbol{y})|\leq \varepsilon,
    \end{equation}
    then
    \begin{equation}\label{lemma2:eq1}
        \sup_{\boldsymbol{x},\boldsymbol{y} \in S_n}\frac{|\tilde\phi(\boldsymbol{x})^{\top} \tilde\phi(\boldsymbol{y}) - k(\boldsymbol{x},\boldsymbol{y})|}{k(x,y)}\leq \frac{\varepsilon}{k(x,y)}=\frac{\varepsilon}{e^{-\frac{\|\boldsymbol{x}-\boldsymbol{y}\|^2}{2\sigma^2}}}.
    \end{equation}
    Since $\|\boldsymbol{x}\|\leq c$ and $\|\boldsymbol{x}-\boldsymbol{y}\|^2\leq 2c$ , we have
    \begin{equation}\label{lemma2:eq2}
        e^{-\frac{\|\boldsymbol{x}-\boldsymbol{y}\|^2}{2}}\geq 1-\frac{\|\boldsymbol{x}-\boldsymbol{y}\|^2}{2\sigma^2}\geq 1-\frac{c}{\sigma^2}.
    \end{equation}
    By substituting (\ref{lemma2:eq2}) into (\ref{lemma2:eq1}), we obtain
        \begin{equation}
         \sup_{\boldsymbol{x},\boldsymbol{y} \in S_n}\frac{|\tilde\phi(\boldsymbol{x})^{\top} \tilde\phi(\boldsymbol{y}) - k(\boldsymbol{x},\boldsymbol{y})|}{k(\boldsymbol{x},\boldsymbol{y})}\leq \frac{\varepsilon}{e^{-\frac{\|\boldsymbol{x}-\boldsymbol{y}\|^2}{2\sigma^2}}}
         \leq \frac{\varepsilon}{1-\frac{c}{\sigma^2}}=O(\varepsilon).   
    \end{equation}
\end{proof}
Through Lemma. \ref{lemma:relative error of RFF}, we can further derive the following lemma, showing the additive error bound when applying RFF in KPKM.
\begin{lemma}
For RFF map $\tilde\phi:\mathbb{R}^{d}\to\mathbb{R}^{D}$ with $D=\Omega\left(\frac{d}{\varepsilon^2} \log \frac{1}{\varepsilon \sqrt{\delta}}\right)$, with probability at least $1-\delta$, we have
\begin{equation}
     \left|\tilde{\mathcal{L}}_s(\mathbf{W},\rho_n)-\mathcal{L}_s(\mathbf{W},\rho_n)\right|\leq\mathcal{O}( k^{1-1/s}\varepsilon).
\end{equation}
\end{lemma}
\begin{proof}
According to lemma \ref{lemma:relative error of RFF}, we have
\begin{equation}\label{lemma4_first}
    \begin{aligned}
    \left\|\tilde{\phi}(\boldsymbol{x}_{i_0}) - \tilde{\mathbf{\Phi}}\mathbf{W}^{(j)}\right\|^2 
&= <\tilde{\phi}(\boldsymbol{x}_{i_0}),\tilde{\phi}(\boldsymbol{x}_{i_0})>
+\frac{\sum_{i=1}^{n}\sum_{i'=1}^{n}w_{ij}w_{i'j}<\tilde{\phi}(\boldsymbol{x}_{i}),\tilde{\phi}(\boldsymbol{x}_{i'})>}{(\sum_{i=1}^{n}w_{ij})^2}\\&
-2\frac{\sum_{i'=1}^{n}w_{i'j}<\tilde{\phi}(\boldsymbol{x}_{i}),\tilde{\phi}(\boldsymbol{x}_{i_0})>}{\sum_{i=1}^{n}w_{ij}}\\&
\leq (1+\varepsilon)k(\boldsymbol{x}_{i_0},\boldsymbol{x}_{i_0}) 
+\frac{(1+\varepsilon)\sum_{i=1}^{n}\sum_{i'=1}^{n}w_{ij}w_{i'j}k(\boldsymbol{x}_i,\boldsymbol{x}_{i'})}{(\sum_{i=1}^{n}w_{ij})^2}\\&
-2\frac{(1+\varepsilon)\sum_{i'=1}^{n}w_{i'j}k(\boldsymbol{x}_i,\boldsymbol{x}_{i_0})}{\sum_{i=1}^{n}w_{ij}}\\&
=(1+\varepsilon)\left\|\phi(\boldsymbol{x}_{i_0}) - \mathbf{\Phi}\mathbf{W}^{(j)}\right\|^2.
\end{aligned}
\end{equation}
Similarly, we have
\begin{equation}
     \left\|\tilde{\phi}(\boldsymbol{x}_{i_0}) - \tilde{\boldsymbol{\theta}}_j\right\|^2 \geq (1-\varepsilon)\left\|\phi(\boldsymbol{x}_{i_0}) - \boldsymbol\theta_j\right\|^2.
\end{equation}
Hence, we arrive at the following conclusion:
\begin{align}
        \left|\tilde{\mathcal{L}}_s(\mathbf{W},\rho_n)-\mathcal{L}_s(\mathbf{W},\rho_n)\right|&=\frac{1}{n}\left|\tilde{f_s}(\mathbf{W})-f_s(\mathbf{W})\right|\nonumber\\
        &=\frac{1}{n}\left|\sum_{i=1}^n \left(M_s(\tilde\phi(\boldsymbol{x}_i),\mathbf{W}) - M_s(\phi(\boldsymbol{x}_i),\mathbf{W})\right)\right|\nonumber
        \\&\leq \frac{1}{n}\sum_{i=1}^n \left|M_s(\tilde\phi(\boldsymbol{x}_i),\mathbf{W}) - M_s(\phi(\boldsymbol{x}_i),\mathbf{W})\right|\nonumber
        \\&\leq \frac{k^{-1/s}}{n}\sum_{i=1}^n  \sum_{j=1}^k \left|\|\tilde\phi(\boldsymbol{x}_i)-\mathbf{\Phi}\mathbf{W}^{(j)}\|- \|\phi(\boldsymbol{x}_i)-\mathbf{\Phi}\mathbf{W}^{(j)}\|\right|\label{eq:lemma4_1}
        \\&\leq \frac{k^{-1/s}}{n}\sum_{i=1}^n  \sum_{j=1}^k \mathcal{O}(\varepsilon)\left\|\phi(\boldsymbol{x}_i)-\mathbf{\Phi}\mathbf{W}^{(j)}\right\|\label{eq:lemma4_2}
        \\&\leq \frac{k^{-1/s}}{n} \mathcal{O}(\varepsilon) \sum_{i=1}^n  \sum_{j=1}^k  \sum_{i_0=1}^{n}w_{i_0 j} \left\|\phi(\boldsymbol{x}_{i})-\phi(\boldsymbol{x}_{i_0})\right\|\label{eq:lemma4_3}
        \\&\leq \frac{k^{-1/s}}{n} \mathcal{O}(\varepsilon)
        \sum_{i=1}^n  \sum_{j=1}^k  \sum_{i_0=1}^{n}w_{i_0 j}\nonumber
        (\|\phi(\boldsymbol{x}_i)\|+\|\phi(\boldsymbol{x}_{i_0})\|)
        \\&\leq \frac{k^{-1/s}}{n}\mathcal{O}(\varepsilon\cdot n\cdot k)\label{eq:lemma4_4}
        \\& =\mathcal{O}( k^{1-1/s}\varepsilon)\nonumber,
\end{align}
where Eq. (\ref{eq:lemma4_1}) is obtained from $\|M_s(\boldsymbol{x}) - M_s(\boldsymbol{y})\| \leq k^{-1/s} \|\boldsymbol{x}-\boldsymbol{y}\|_1$ \cite{beliakov2010lipschitz}; Eq. (\ref{eq:lemma4_2}) is obtained from Eq. (\ref{lemma4_first}); Eq. (\ref{eq:lemma4_3}) is obtained from $\sum_{i_0=1}^{n}w_{i_0 j}=1$ and triangle inequality; Eq. (\ref{eq:lemma4_4}) is obtained from $\sum_{i_0=1}^{n}w_{i_0 j}=1$ and $\|\phi(\boldsymbol{x})\|=1$.
\end{proof}
Now we can bound $A_2$ by the following lemma.
\begin{lemma}\label{lemma:A2}
    For RFF map $\tilde\phi:\mathbb{R}^{d}\to\mathbb{R}^{D}$ with $D=\Omega\left(\frac{d}{\varepsilon^2} \log \frac{1}{\varepsilon \sqrt{\delta}}\right)$, with probability at least $1-\delta$, we have
    \begin{equation}
    A_2=\left|\mathcal{L}_s(\tilde{\mathbf{W}}_{n,s},\rho_n)-\mathcal{L}_s(\mathbf{W}_{n,s},\rho_n)\right|\leq\mathcal{O}(k^{1-1/s}\varepsilon)
    \end{equation}
\end{lemma}
\begin{proof}
        \begin{equation}
    \begin{aligned}
A_2&=\left|\mathcal{L}_s(\tilde{\mathbf{W}}_{n,s},\rho_n)-\mathcal{L}_s(\mathbf{W}_{n,s},\rho_n)\right|
    =\mathcal{L}_s(\tilde{\mathbf{W}}_{n,s},\rho_n)-\mathcal{L}_s(\mathbf{W}_{n,s},\rho_n)
\\&=\left(\mathcal{L}_s(\tilde{\mathbf{W}}_{n,s},\rho_n)-\tilde{\mathcal{L}}_s(\tilde{\mathbf{W}}_{n,s},\rho_n)\right)+\left(\tilde{\mathcal{L}}_s(\tilde{\mathbf{W}}_{n,s},\rho_n)-\mathcal{L}_s(\mathbf{W}_{n,s},\rho_n)\right)
     \\&\leq \left(\mathcal{L}_s(\tilde{\mathbf{W}}_{n,s},\rho_n)-\tilde{\mathcal{L}}_s(\tilde{\mathbf{W}}_{n,s},\rho_n)\right)+\left(\tilde{\mathcal{L}}_s(\mathbf{W}_{n,s},\rho_n)-\mathcal{L}_s(\mathbf{W}_{n,s},\rho_n)\right)
     \\&\leq \mathcal{O}( k^{1-1/s}\varepsilon).
    \end{aligned}
\end{equation}
\end{proof}

\paragraph{Remark: }
If we want ${A}_2$ arrive at $\mathcal{O}\left(\sqrt{{k^{3-2/s}}/{n}}\right)$, we need to set $\varepsilon = \sqrt{{k}/{n}}$, and then the corresponding dimensionality of RFF is $\Omega(\frac{nd\log(n/k\delta)}{k})$. The proof of Theorem \ref{theorem: theorem1} can be completed by combining the upper bounds of $A_1, A_2, A_3, A_4$.

\subsection{A.2 Proof of Theorem 2}
Toward proving strong consistency, we introduce the following Uniform Strong Law of Large Numbers (USLLN), which plays a pivotal role in the proof of our main theorem.
\begin{lemma}[USLLN]\label{lemma:USLLN}\cite{paul2022implicit}
    For $s_0\leq-1$,
    \begin{equation}
        \sup_{s \leq s_0, \mathbf{W} \in \mathcal{A}} \left| \mathcal{L}_s(\mathbf{W},\rho_n) - \mathcal{L}_s(\mathbf{W},\rho) \right| \to 0
    \end{equation}
    almost surely under $\rho$.
\end{lemma}
\begin{theorem}
    Under Assumption 1, $\tilde{\mathbf{W}}_{n,s} \xrightarrow{\text{a.s.}} \mathbf{W}^*$ with any constant probability when $n\to \infty$, RFF dimensionality $D\to \infty$ and $s\to -\infty$. 
\end{theorem}

\begin{proof}
    We must show for arbitrarily small $r>0$, that the minimizer of KPK with RFF $\tilde{\mathbf{W}}_{n,s}$ eventually lies inside the ball.  Under assumption 1, it suffices to show that for all $\varepsilon>0$, there exists $N_1>0, N_2>0$ and $N_3<0$ such that $n>N_1, D>N_2$ and $s<N_3$ implies that $\Psi(\tilde{\mathbf{W}}_{n,s}, \rho) - \Psi(\mathbf{W}^*, \rho) \leq \varepsilon$ with any constant probability.
    We start by breaking down the target formula into four terms:
    \begin{equation}
        \begin{aligned}
            \Psi(\tilde{\mathbf{W}}_{n,s}, \rho) &- \Psi(\mathbf{W}^*, \rho)=E_1+E_2+E_3+E_4, \ \ \text{where}\\ &
            E_1 =  \Psi(\tilde{\mathbf{W}}_{n,s}, \rho)-\mathcal{L}_s(\tilde{\mathbf{W}}_{n,s}, \rho);\\ &
            E_2 =  \mathcal{L}_s(\tilde{\mathbf{W}}_{n,s}, \rho)-\mathcal{L}_s(\tilde{\mathbf{W}}_{n,s}, \rho_n);\\ &
            E_3 =  \mathcal{L}_s(\tilde{\mathbf{W}}_{n,s}, \rho_n)-\mathcal{L}_s({\mathbf{W}}_{n,s}, \rho_n);\\ &
            E_4 =  \mathcal{L}_s({\mathbf{W}}_{n,s}, \rho_n)-\Psi(\mathbf{W}^*, \rho).\\
        \end{aligned}
    \end{equation}
    First, because $M_s(\phi(\boldsymbol{x}),\mathbf{W})\to \min_{j\in[k]}\|\phi(\boldsymbol{x})-\mathbf{\Phi}\mathbf{W}\|^2 $ when $s\to-\infty$, we can choose $N_3<0$ such that if $s<N_3$, then
    \begin{equation}
    \begin{aligned}
            E_1&=\Psi(\tilde{\mathbf{W}}_{n,s}, \rho)-\mathcal{L}_s(\tilde{\mathbf{W}}_{n,s}, \rho)
            =\int \left( \min_{j\in[k]}\|\phi(\boldsymbol{x})-\mathbf{\Phi}\tilde{\mathbf{W}}^{(j)}_{n,s}\|^2-M_s(\phi(\boldsymbol{x}),\tilde{\mathbf{W}}_{n,s}) \right)d\rho
            \leq\frac{\varepsilon}{6}\int d\rho = \frac{\varepsilon}{6}.
    \end{aligned}
    \end{equation}
    Then, according to lemma \ref{lemma:USLLN}, we can choose $N_1$ large enough such that $n\geq N_1$ implies that $E_2\leq\varepsilon/6$. According to lemma \ref{lemma:A2}, when $D=\Omega(\frac{dk^{2-2/s}}{\varepsilon^2}\log\frac{k^{1-1/s}}{\varepsilon\sqrt{\delta}})$, with probability at least $1-\delta$ we have $E_3\leq \mathcal{O}(\varepsilon)$, which implies we can choose $N_2$ large enough such that $E_3\leq\varepsilon/3$ when $D\geq N_2$. To bound $E_4$, we observe that
        \begin{align}
            E_4 &= \mathcal{L}_s({\mathbf{W}}_{n,s}, \rho_n)-\Psi(\mathbf{W}^*, \rho)\nonumber\\&
            \leq \mathcal{L}_s({\mathbf{W}^*}, \rho_n)-\Psi(\mathbf{W}^*, \rho)\nonumber\\&
            \leq \mathcal{L}_s({\mathbf{W}^*}, \rho)+\varepsilon/6-\Psi(\mathbf{W}^*, \rho)\label{eq:theorem2_1}\\&
            \leq \varepsilon/6 + \varepsilon/6 = \varepsilon/3,\label{eq:theorem2_2}
        \end{align}
    where Eq. (\ref{eq:theorem2_1}) is obtained from lemma \ref{lemma:USLLN}. Eq. (\ref{eq:theorem2_2}) can be obtained by a similar analysis of $E_1$. In conclusion, we have
    \begin{equation}
    \begin{aligned}
             \Psi(\tilde{\mathbf{W}}_{n,s}, \rho) - \Psi(\mathbf{W}^*, \rho)&=E_1+E_2+E_3+E_4\\&  
             \leq \varepsilon/6 +\varepsilon/6 +\varepsilon/3 +\varepsilon/3 \\&= \varepsilon.
    \end{aligned}
    \end{equation}
\end{proof}

\subsection{A.3 Proof of Theorem 3}
To prove Theorem 3, we first introduce the following lemma, which discusses the relative error bound for approximating kernel $k$-means through RFF. Given a data set $P \subset \mathbb{R}^d$ and kernel function $k : \mathbb{R}^d \times \mathbb{R}^d \to \mathbb{R}$, denoting the feature mapping as $\varphi : \mathbb{R}^d \to \mathcal{H}$. The kernel $k$-clustering problem asks for a $k$-partition $\mathcal{C} = \{C_1, C_2, ..., C_k\}$ of $P$ that minimizes the cost function:
\begin{equation}
    \mathrm{cost}_p^\phi(P, \mathcal{C}) := \sum_{i=1}^{k} \min_{c_i \in \mathcal{H}} \sum_{x \in C_i} \| \phi(x) - c_i \|_2^2.
\end{equation}

\begin{lemma}\cite{cheng2023relative}\label{lemma:RFF for kmeans}
 For kernel $k$-clustering  problem whose kernel function $K:\mathbb{R}^d \times\mathbb{R}^d \to \mathbb{R}$ is shift-invariant and analytic at the origin, for every data set $P \subset \mathbb{R}^d$, the RFF map $\pi:\mathbb{R}^{d}\to\mathbb{R}^{D}$ with target dimension $D = \Omega(( \log^3 \frac{k}{\delta} + \log^3 \frac{1}{\varepsilon} + 2^8)/\varepsilon^2)$ satisfies
 \begin{equation}
     \begin{aligned}
            \Pr[\forall \text{k-partition } \mathcal{C} \text{ of } P : \operatorname{cost}^{\pi}(P,\mathcal{C}) &\in (1 \pm \varepsilon) \cdot {\operatorname{cost}^{\phi}(P,\mathcal{C})} ] \geq 1-\delta. 
     \end{aligned}
 \end{equation}
\end{lemma}

\begin{proof}
    First we note that when $n,s\to \infty$ we have 
    \begin{equation}
        \tilde\Psi(\tilde{\mathbf{W}}_{n,s}, \rho) \to \tilde\Psi(\tilde{\mathbf{W}}^*, \rho).
    \end{equation}
    This can be proved by breaking down the target formula into three terms:
    \begin{equation}
    \begin{aligned}
            \tilde\Psi(\tilde{\mathbf{W}}_{n,s}, \rho) &-\tilde\Psi(\tilde{\mathbf{W}}^*, \rho)=E_1+E_2+E_3, \ \ \text{where}\\&
            E_1 =  \tilde\Psi(\tilde{\mathbf{W}}_{n,s}, \rho)-\tilde{\mathcal{L}}_s(\tilde{\mathbf{W}}_{n,s}, \rho);\\&
            E_2 =  \tilde{\mathcal{L}}_s(\tilde{\mathbf{W}}_{n,s}, \rho)-\tilde{\mathcal{L}}_s(\tilde{\mathbf{W}}_{n,s}, \rho_n);\\&
            E_3 =  \tilde{\mathcal{L}}_s({\mathbf{W}}_{n,s}, \rho_n)-\tilde{\Psi}(\tilde{\mathbf{W}}^*, \rho).\\
    \end{aligned}
    \end{equation}

    Then the conclusion can be obtained by a similar analysis of Theorem 2. According to the assumption, we now have $\lim_{n,s\to\infty}\tilde{\mathbf{W}}_{n,s}=\tilde{\mathbf{W}}^{*}$. Because $\mathbf{W}^{*}$ and $\tilde{\mathbf{W}}^{*}$ minimize the $k$-means problem, they can be seen as a partition of data.
    According to Lemma \ref{lemma:RFF for kmeans}, with probability at least $1-\delta$ we have
    \begin{equation}
        \begin{aligned}
            \lim_{n,s\to\infty}(1-\varepsilon)\Psi(\tilde{\mathbf{W}}_{n,s},\rho)&=
            (1-\varepsilon)\Psi(\tilde{\mathbf{W}}^{*},\rho)\\
            &\leq \tilde\Psi(\tilde{\mathbf{W}}^{*},\rho)\\
            &\leq \tilde\Psi({\mathbf{W}}^{*},\rho)\\
            &\leq (1+\varepsilon)\Psi({\mathbf{W}}^{*},\rho).
        \end{aligned}
    \end{equation}
\end{proof}

\newpage
\section{B. Supplementary Experiments}
In this section, we first introduce the parameter settings and datasets. Then we supplement the experimental results for the proposed single-view clustering method RFF-KPKM. Subsequently, we supplement the parameter sensitivity and convergence results of IP-RFF-MKPKM on all datasets.

\subsection{B.1 Parameter Settings}
\paragraph{Parameter Settings for Kernel and RFF}
For all kernel clustering methods we referenced, we configured their parameters within their recommended ranges. We all choose Gaussian kernel  $k(\boldsymbol{x}_{i},\boldsymbol{x}_{j})=\exp(-\|\boldsymbol{x}_i-\boldsymbol{x}_j\|^2/2\sigma^2)$. The bandwidth parameter $\sigma$ for the Gaussian kernel is chosen as $10^{3}$.
Correspondingly, to generate RFF, the frequency vectors $\boldsymbol{\omega}_{i}$ should be drawn i.i.d from Gaussian distribution $\mathcal{N}(\boldsymbol{0},\sigma^{-2}\mathbf{I}_{d})$. According to Theorem \ref{theorem:approx consistency}, we could set the RFF dimension $D$ to  $\Omega(( \log^3 \frac{k}{\delta} +  \log^3 \frac{1}{\varepsilon} + 2^8)/\varepsilon^2)$ to ensure the ($1+\varepsilon$) approximation with probability at least $1-\delta$. In the experiments, we set $\varepsilon$ and $\delta$ to $0.5$, and set $D$ to $\lceil(4 \log^3 \frac{k}{\delta})/\varepsilon^2\rceil$ to ensure the approximation property.

\paragraph{Parameter Settings for the Proposed Methods}
For proposed IP-RFF-MKPKM, we set $\eta_{j,l}=\sum_{i=1}^{n}\|\tilde{\phi}_{l}(\boldsymbol{x}_{i})-\boldsymbol{\theta}_{j,l}\|^{2}/n$ according to the suggestion in \cite{krishnapuram1993possibilistic}. For the entropy regularization parameter $\lambda$, we search for the best one in the set $\{10^{-1},10^{0},\dots,10^{3}\}$. We also search for the optimal initial value $s_0$ of $s$ among the set $\{5, 10,\dots, 25\}$.  To initialize the cluster centroids $\mathbf{\Theta}$, we obtain them by performing $k$-means on the mapped data. Following the approach in \cite{paul2022implicit}, we fix the annealing parameter $\gamma=1.04$ and set the fuzziness parameter $m=2$.  The tuning and choice of $\lambda$, $s_0$, $\gamma$, $m$ are the same for both IP-RFF-MKPKM and RFF-KPKM. In IP-RFF-MKPKM experiments, we update $s=\gamma\cdot s$ every two iterations. In RFF-KPKM experiments, we update $s=\gamma\cdot s$ every three iterations.

\subsection{B.2 Datasets} To evaluate the effectiveness of the proposed RFF-KPKM, we use six single-view datasets derived from the first view of six multi-view datasets: Yale \cite{cai2005using}, Caltech101-20 \cite{li2015large}, 100leaves\cite{wang2019gmc}, CIFAR10 \cite{krizhevsky2009learning}, YTF10 and YTF20 \cite{wolf2011face}. 
\begin{table}[h]
\centering
\begin{tabular}{cccc}
\toprule
Dataset      & Size  &Features  &Classes   \\ \midrule
Yale 
& 165 & 4096 & 15    
\\
100leaves  
& 1600 & 64 & 100    
\\
Caltech101-20  
& 2386 & 48& 20   
\\
YTF10 
& 38654 & 944 & 10   
\\
CIFAR10 
& 50000 & 512 & 10    
\\
YTF20
& 63896 & 944 & 20     
\\
\bottomrule
\end{tabular}
\caption{Datasets used in single-view experiments.}\label{tab:single view datasets}
\end{table}

To evaluate the effectiveness of the proposed IP-RFF-MKPKM, we use seven widely used multi-view datasets: LGG \cite{cancer2015comprehensive}, Caltech101-7 \cite{li2015large}, HW2 \cite{Kevin2013}, NUS-WIDE-SCENE \cite{chua2009nus}, NUS-WIDE-OBJECT \cite{chua2009nus} and CIFAR10 \cite{krizhevsky2009learning}.
\begin{table}[ht]
\centering
\begin{tabular}{cccc}
\toprule
Dataset      & Size  &Features  &Classes   \\ \midrule
LGG  
& 267 & 2000/$\dots$/209& 3  
\\
Caltech101-7  
& 1474 & 48/$\dots$/928& 7   
\\
HW2 
& 2000 & 240/216 & 10    
\\
Caltech101-20 
& 2386 & 48/$\dots$/928 & 20     
\\
NUS-WIDE-SCENE
& 4095 & 128/$\dots$/64 & 33     
\\
NUS-WIDE-OBJECT
& 23953 & 129/$\dots$/65 & 31
\\
CIFAR10
& 60000 & 944/$\dots$/640 & 10
\\
\bottomrule
\end{tabular}
\caption{Datasets used in multi-view experiments.}\label{tab:multi view datasets}
\end{table}

\newpage
\subsection{B.3 Experimental Results for RFF-KPKM}
Table \ref{tab:result comparison1} displays the clustering outcomes of RFF-KPKM on the six benchmark datasets. From the results, we draw the following conclusions:
\begin{enumerate}
    \item Our proposed KPKM algorithm generally outperforms existing clustering methods on most datasets. In terms of ACC, KPKM surpasses the second-best algorithm on the Yale, Caltech101-20, 100leaves, CIFAR10, YTF10, YTF20 datasets by margins of 3.03\%, 2.44\%, 2.07\%, 0.38\%, 10.81\%, 4.12\%, respectively. 
    \item When compared to the original KPKM method, our approach generally achieves superior performance. On datasets Yale, Caltech101-20, and 100leaves, IP-RFF-MKPKM consistently outperforms KPKM by 11.51\%, 2.60\%, and 2.32\% in terms of ACC. This result demonstrates that $\mathrm{poly}(\varepsilon^{-1}\log k)$ RFF can preserve the accuracy of KPKM, while the combination with the PKM method can improve the clustering result.
    \item RFF-KPKM can run on the large-scale datasets CIFAR10, YTFace, YTF20, containing 63896 instances at most, while KPKM can not due to the memory limitation. The result shows the strong scalability of our method.
\end{enumerate}
\begin{table*}[ht]
\centering
\scalebox{0.95}{\begin{tabular}{ccccccc}
\hline
Dataset    &Metric      & Kernel $k$-means        & Spectral Clustering        & Power $k$-means      & Kernel Power $k$-means & Proposed \\ 

\hline
                               & ACC          &58.18	&41.67	&57.58	&49.70 	&\textbf{61.21}
                               \\
                               & NMI         &61.49	&47.00 	&65.90 	&58.33	&\textbf{68.18}
                              \\
                            \multirow{-3}{*}{Yale}  & Purity      &25.39	&44.00 	&59.39	&52.73	&\textbf{61.21}
                        \\ \hline
                               & ACC           &23.72	&23.68	&25.31 	&25.15	&\textbf{27.75}
                                 \\
                               & NMI       &29.46	&30.05	&29.56	&27.88	&\textbf{29.71}
                                 \\
                        \multirow{-3}{*}{Caltech101-20}  & Purity &53.56	&53.02	&52.14	&52.10	&\textbf{53.98}
                                \\ \hline
                               & ACC          &60.19	&32.86	&62.06 	&61.81	&\textbf{64.13}
                       \\
                               & NMI          &81.45	&62.32	&82.75	&80.87	&\textbf{82.89}
                           \\
\multirow{-3}{*}{100leaves}      & Purity       &64.06	&36.81	&65.88	&65.00	&\textbf{66.88}
                        \\ \hline
                               & ACC          &87.36	&-	&88.06	&-	&\textbf{88.44}
                        \\
                               & NMI          &78.28	&-	&\textbf{79.02}	&-	&78.65
                        \\
                            \multirow{-3}{*}{CIFAR10}         & Purity        &87.36	&-	&88.06	&-	&\textbf{88.44}
                           \\ \hline
                               & ACC        &72.04	&-	&76.33	&-	&\textbf{87.14}
                        \\
                     & NMI       &76.95	&-	&80.71	&-	&\textbf{81.86}
                        \\
\multirow{-3}{*}{YTF10}     & Purity       &76.39	&-	&81.35	&-	&\textbf{87.14}
                           \\ \hline
                               & ACC         &69.01	&-	&64.47	&-	&\textbf{73.13}
                        \\
    & NMI          &77.41	&-	&72.96	&-	&\textbf{78.77}
                          \\
\multirow{-3}{*}{YTF20}   & Purity       &73.51	&-	&68.35	&-	&\textbf{78.34}
           \\ \hline
\end{tabular}}
    \caption{ Empirical evaluation and comparison of RFF-MKPKM with four baseline methods on six benchmark datasets in terms of
clustering accuracy (ACC), normalized mutual information (NMI), and Purity. The best value is marked in bold, and `-' indicates that the algorithm failed to run due to insufficient memory.}
    \label{tab:result comparison1}
\end{table*}

\subsection{B.4 Supplementary Convergence and Sensitivity Analysis for IP-RFF-MKPKM}
We give the complete convergence behavior of IP-RFF-MKPKM on six benchmark datasets as shown in Fig. \ref{Convergence analysis of IP-RFF-MKPKM on six benchmark datasets.}.
We also study the sensitivity of the parameters $s_0$ and $\lambda$ as shown in Fig. \ref{Sensitivity analysis}.

\begin{figure}[htbp]
	\centering
	\begin{subfigure}{0.3\linewidth}
		\centering
		\includegraphics[width=0.8\linewidth]{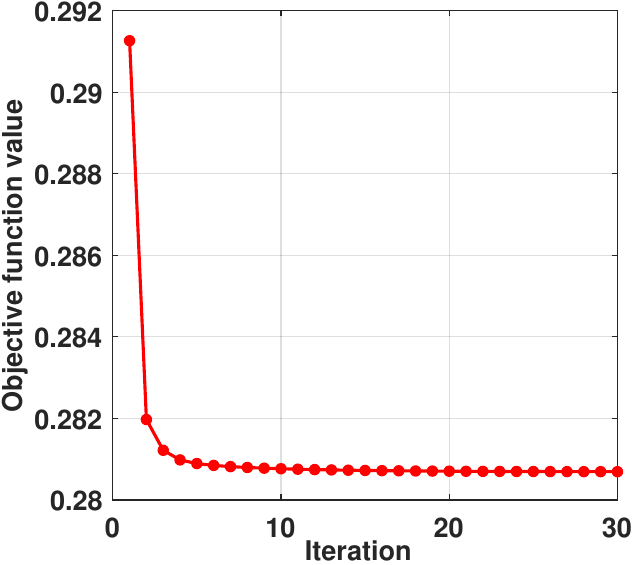}
		\caption{LGG}
	\end{subfigure}
	\centering
	\begin{subfigure}{0.3\linewidth}
		\centering
		\includegraphics[width=0.8\linewidth]{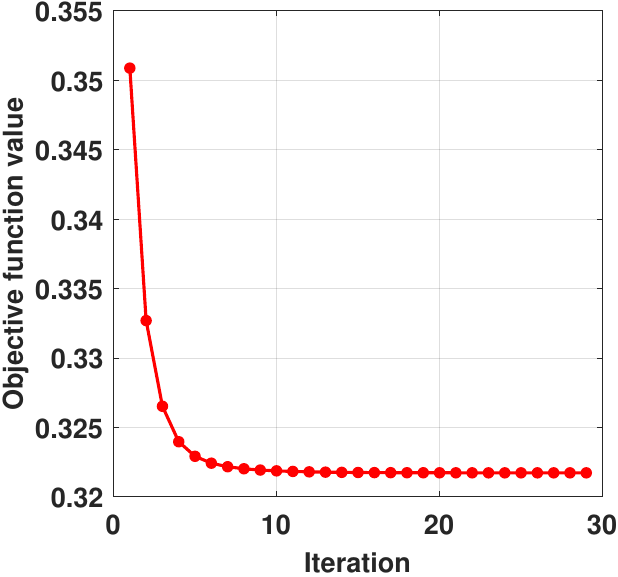}
		\caption{HW2}
	\end{subfigure}
        \centering
	\begin{subfigure}{0.3\linewidth}
		\centering
		\includegraphics[width=0.8\linewidth]{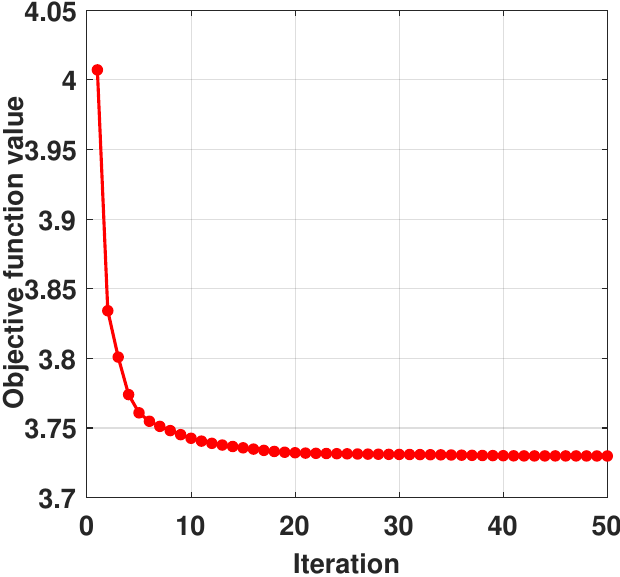}
		\caption{Caltech101-20}
	\end{subfigure}
    	\centering
	\begin{subfigure}{0.3\linewidth}
		\centering
		\includegraphics[width=0.8\linewidth]{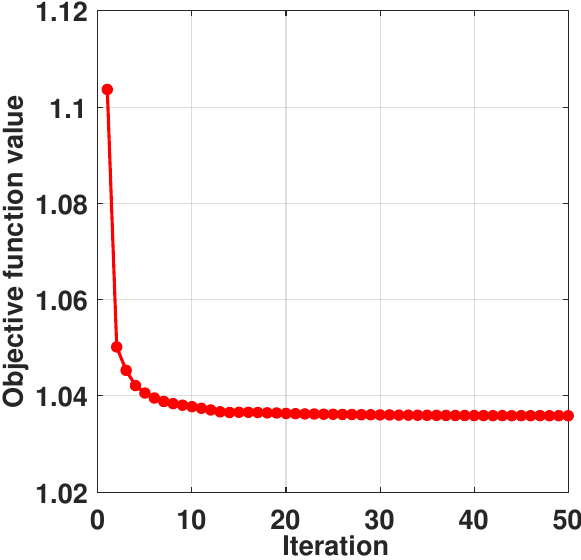}
		\caption{NUS-WIDE-SCENE}
	\end{subfigure}
	\centering
	\begin{subfigure}{0.3\linewidth}
		\centering
		\includegraphics[width=0.8\linewidth]{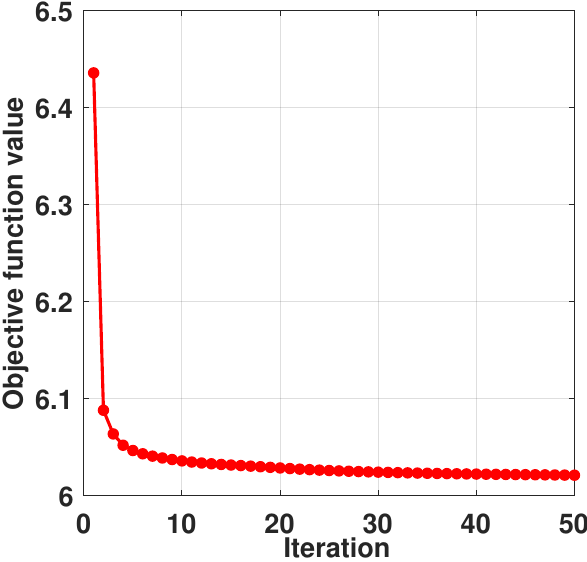}
		\caption{NUS-WIDE-OBJECT}
	\end{subfigure}
        \centering
	\begin{subfigure}{0.3\linewidth}
		\centering
		\includegraphics[width=0.8\linewidth]{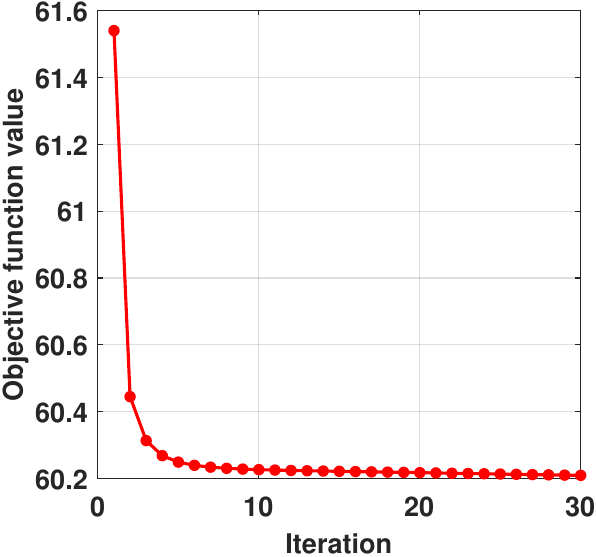}
		\caption{Cifar10}
	\end{subfigure}
	\caption{Convergence analysis of IP-RFF-MKPKM on six benchmark datasets.}
        \label{Convergence analysis of IP-RFF-MKPKM on six benchmark datasets.}
\end{figure}

\begin{figure}[htbp]
	\centering
	\begin{subfigure}{0.3\linewidth}
		\centering
		\includegraphics[width=0.9\linewidth]{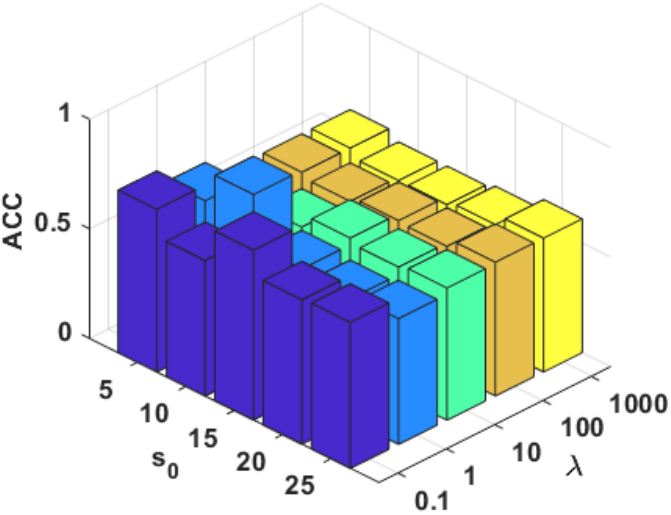}
		\caption{LGG}
	\end{subfigure}
	\centering
	\begin{subfigure}{0.3\linewidth}
		\centering
		\includegraphics[width=0.9\linewidth]{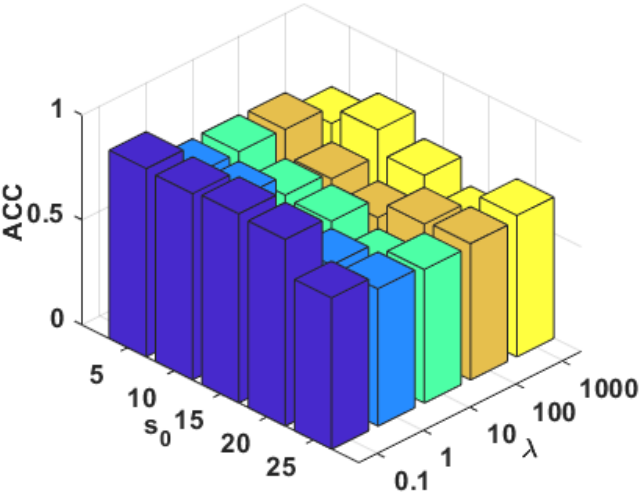}
		\caption{HW2}
	\end{subfigure}
        \centering
	\begin{subfigure}{0.3\linewidth}
		\centering
		\includegraphics[width=0.9\linewidth]{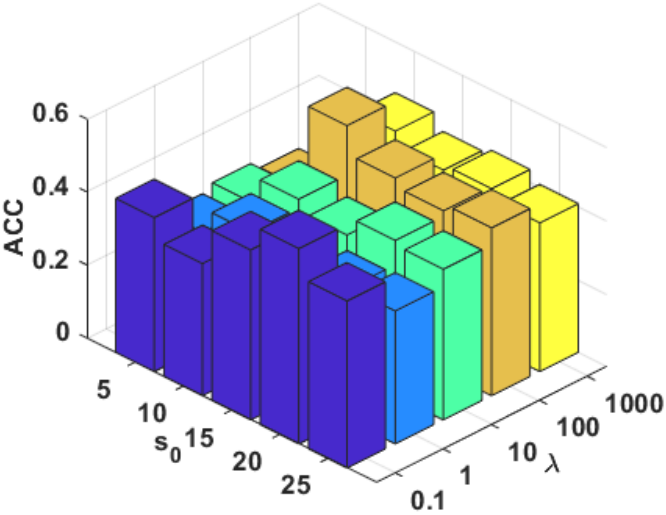}
		\caption{Caltech101-20}
	\end{subfigure}
    	\centering
	\begin{subfigure}{0.3\linewidth}
		\centering
		\includegraphics[width=0.9\linewidth]{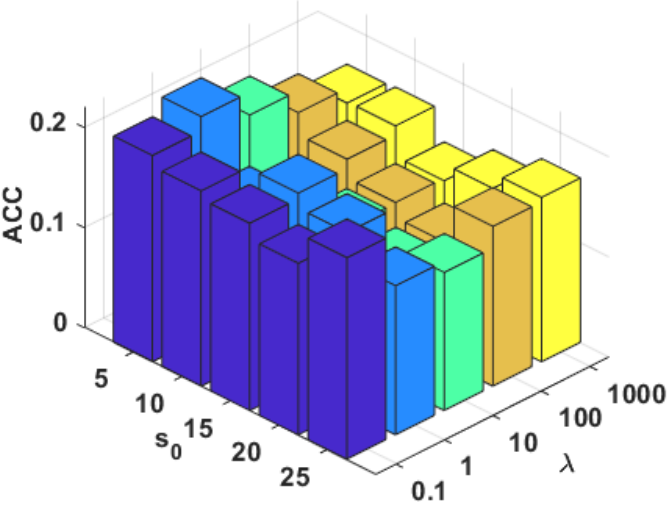}
		\caption{NUS-WIDE-SCENE}
	\end{subfigure}
	\centering
	\begin{subfigure}{0.3\linewidth}
		\centering
		\includegraphics[width=0.9\linewidth]{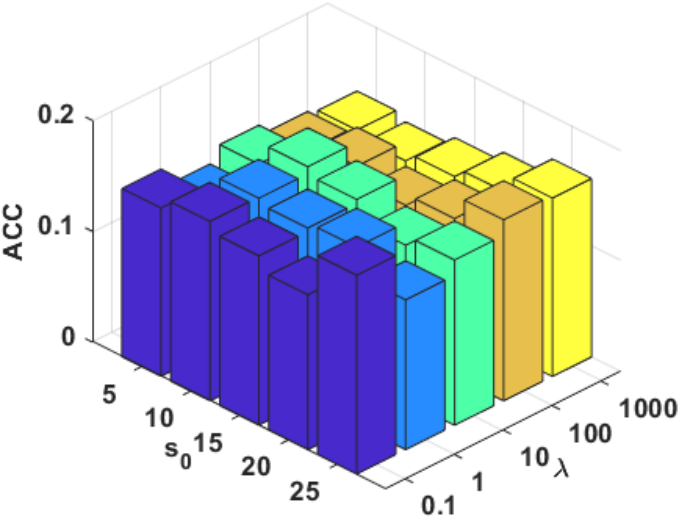}
		\caption{NUS-WIDE-OBJECT}
	\end{subfigure}
        \centering
	\begin{subfigure}{0.3\linewidth}
		\centering
		\includegraphics[width=0.9\linewidth]{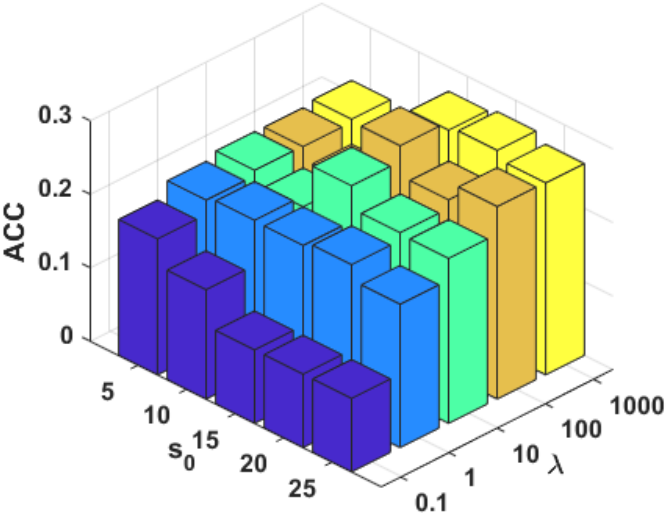}
		\caption{Cifar10}
	\end{subfigure}
	\caption{Sensitivity analysis of $s_0$ and $\lambda$ of IP-RFF-MKPKM on six benchmark datasets.}
        \label{Sensitivity analysis}
\end{figure}

\newpage
\section{C. Missing Pseudocode}

To elucidate the implementation details and operational workflow of our algorithm, we provide the pseudocode of RFF-KPKM as described in Alg. \ref{alg:RFF-KPKM} and IP-RFF-MKPKM as shown in Alg. \ref{alg:IP-RFF-MKPK}.

\begin{algorithm}
    \caption{RFF-KPKM}
    \label{alg:RFF-KPKM}
    \begin{algorithmic}[1]
        \REQUIRE $s_0<0$, $\mathbf{\Theta}_0$, dataset $\mathbf{X}\in\mathbb{R}^{n\times d}$, RFF dimensionality $D$ and constant $\gamma\geq1$.
        \STATE $\tilde{\mathbf{\Phi}}\gets \tilde{\mathbf{\Phi}}=(\tilde{\phi}(\boldsymbol{x}_1),\tilde{\phi}(\boldsymbol{x}_2),\cdots,\tilde{\phi}(\boldsymbol{x}_n))$, $\tilde{\phi}:\mathbb{R}^{d}\to\mathbb{R}^{D}$ is RFF map. 
        \REPEAT
            \STATE $w_{ij}^{(t)}\gets(\sum_{l=1}^{k}\|\tilde{\phi}(\boldsymbol{x}_i)-\boldsymbol{\theta}_{l}^{(t)}\|^{2s_t})^{\frac{1}{s_t}-1}\|\tilde{\phi}(\boldsymbol{x}_i)-\boldsymbol{\theta}_{j}^{(t)}\|^{2(s_t-1)}$
            \STATE $\mathbf{W}_{ij}^{(t)}\gets w_{ij}^{(t)}/\sum_{i_0=1}^{n}w_{i_0j}^{(t)}$
            \STATE $\boldsymbol{\theta}_{j}^{(t+1)}\gets\tilde{\mathbf{\Phi}} {{\mathbf{W}}^{(t)}}^{(j)}$
            \STATE (Optional) $s_{t+1}\gets\gamma \cdot s_t$ 
        \UNTIL convergence
    \end{algorithmic}
\end{algorithm}

\begin{algorithm}
    \caption{IP-RFF-MKPKM}
    \label{alg:IP-RFF-MKPK}
    \begin{algorithmic}[1]
        \REQUIRE $\mathbf{X}\in\mathbb{R}^{n\times d}$, $\mathbf{U}_{0}$, $\mathbf{\Theta}_0$,  $\boldsymbol{\alpha}_{0}$, $D$ , $s_0$, $m\geq 2$, $\eta\geq1$.
        \ENSURE $\mathbf{U}^{*}$, $\mathbf{\Theta}^{*}$, $\boldsymbol{\alpha}^{*}$
        \STATE $\tilde{\mathbf{\Phi}}_{l}\gets \tilde{\mathbf{\Phi}}_{l}=(\tilde{\phi}_{l}(\boldsymbol{x}_1),\tilde{\phi}_{l}(\boldsymbol{x}_2),\cdots,\tilde{\phi}_{l}(\boldsymbol{x}_n))$, $\tilde{\phi}_{l}:\mathbb{R}^{d}\to\mathbb{R}^{D}$ is RFF map. 
        \STATE $\tilde{d}_{ij,l}^{(0)} \gets (u_{ij}^{(0)})^{m}\|\tilde\phi_l(\boldsymbol{x}_i)-\boldsymbol{\boldsymbol{\theta}}_{j,l}^{(0)}\|^2
     +(1-u_{ij}^{(0)})^{m}\eta_{j,l}.$
        \REPEAT
            \STATE \textbf{Step1:} Compute the partial derivative by $$w_{ij}^{(t)}\gets \frac{\frac{1}{k}(\sum_{l=1}^{L}\alpha_l^{(t)}\tilde{d}_{ij,l}^{(t)})^{(s-1)}}
{(\frac{1}{k}\sum_{c=1}^k (\sum_{l=1}^{L}\alpha_l^{(t)}\tilde{d}_{ic,l}^{(t)})^{s})^{(1-1/s)}}.$$
            \STATE \textbf{Step2:} Update $\mathbf{U}$ by
            $$\mathbf{U}^{(t)}\gets \left(1 + \left(\frac{\sum_{l=1}^{L}\alpha_l^{(t)}\|\tilde\phi_l(\boldsymbol{x}_i)-{\boldsymbol{\theta}}_{j,l}^{(t)}\|^2}{\sum_{l=1}^{L}\alpha_l^{(t)}\eta_{j,l}}\right)^{\frac{1}{m - 1}}\right)^{-1}.$$
            \STATE \textbf{Step3:} Update $\mathbf{\Theta}$ by
            $$\boldsymbol{\theta}_{j,l}^{(t+1)} = \frac{\sum_{i=1}^{n} w_{ij}^{(t)}\left(u_{ij}^{(t+1)}\right)^m \tilde\phi_l(\boldsymbol{x}_i)}{\sum_{i=1}^{n} w_{ij,l}^{(t)}\left(u_{ij}^{(t+1)}\right)^m}.$$
            \STATE \textbf{Step4:} Compute $\tilde{d}_{ij,l}$ by
            $$\tilde{d}_{ij,l}^{(t)} \gets (u_{ij}^{(t)})^{m}\|\tilde\phi_l(\boldsymbol{x}_i)-\boldsymbol{\boldsymbol{\theta}}_{j,l}^{(t)}\|^2
     +(1-u_{ij}^{(t)})^{m}\eta_{j,l}.$$
            \STATE \textbf{Step5:} Update $\boldsymbol{\alpha}$ by
            $$    \alpha_l\gets\frac{\exp\left(-\frac{1}{\lambda} \sum_{i=1}^{n} \sum_{j=1}^{k} w_{ij}^{(t)} \tilde{d}_{ij,l}^{(t+1)}\right)}{\sum_{l=1}^{L} \exp\left(-\frac{1}{\lambda} \sum_{i=1}^{n} \sum_{j=1}^{k} w_{ij}^{(m)} \tilde{d}_{ij,l}^{(t+1)}\right)}.$$
            \STATE \textbf{Step6(Optional):} $s_{t+1}\gets\eta \cdot s_t.$
        \UNTIL convergence
    \end{algorithmic}
\end{algorithm}

\end{document}